\documentclass{article}
\usepackage{PRIMEarxiv}

\usepackage[utf8]{inputenc} 
\usepackage[T1]{fontenc}    
\usepackage[colorlinks=true,
            citecolor=blue,
            linkcolor=blue,
            urlcolor=blue]{hyperref}
\usepackage{url}            
\usepackage{nicefrac}       
\usepackage{microtype}      
\usepackage{lipsum}
\usepackage{fancyhdr}       
\usepackage{graphicx}       
\usepackage{subcaption}     
\graphicspath{{imgs/}}      

\usepackage{amsmath,amsthm,amssymb,amsfonts,mathtools}
\usepackage{booktabs,array,tabularx}
\usepackage{algorithm2e,algorithmic}
\usepackage{enumitem}
\usepackage[sort,numbers]{natbib}
\usepackage{xcolor}         
\usepackage{float}
\usepackage[title,titletoc]{appendix}
\usepackage[section]{placeins}

\SetKwInput{KwInit}{Initialization}
\SetKwInput{KwEvol}{Evolution}
\SetKwInput{KwOptional}{Optional input}
\DontPrintSemicolon

\usepackage{libertine}
\newtheorem{theorem}{Theorem}
\newtheorem*{theorem*}{Theorem}
\newtheorem{lemma}{Lemma}

\RestyleAlgo{ruled}

\title{
    Composing and Steering Diffusion Models with Generalized Feynman-Kac Correctors
}

\author{
    \large
    Wael Mattar \thanks{School of Mathematical Sciences, Tel Aviv University, Israel. \\ Emails: \texttt{waelmattar.95@gmail.com};\quad \texttt{d.e.freedman@gmail.com}}
    \and
    \large
    \textbf{Daniel Freedman }\footnotemark[1]
    }

\begin{document}

\maketitle

\begin{abstract}
Diffusion models have become the leading paradigm in generative modeling, yet many applications require sampling from distributions that are not those of any single pretrained model, but are instead defined through combinations of several models, possibly tilted toward high values of a downstream reward. Naively combining the scores of the models yields incorrect samples, while existing corrections are restricted to specific families of targets. In this work, we introduce the generalized Feynman--Kac correctors (GFKC), a principled framework for sampling from an arbitrary target function of several pretrained diffusion models. We show that the geometric interpolation between an arbitrary initial function and the target function of the models satisfies a Feynman--Kac partial differential equation, and derive its potential in closed form. To evaluate this potential during generation, we further generalize the log-probability ODE to track the density of each model along trajectories driven by a different velocity field. The resulting algorithm operates entirely in the log domain, incorporates reward-based steering at no additional conceptual cost, and corrects the generation on the fly, without retraining or fine-tuning the models. We validate the method on analytic Gaussian models, where it closely matches the ground-truth targets, and on diffusion models trained on MNIST, where it samples from non-trivial combinations of the models and enables effective steering.
\end{abstract}

\section{Introduction}

Diffusion models~\cite{sohl2015deep,ho2020denoising,song2021score} have become the leading paradigm in generative modeling. By learning to reverse a gradual noising process, they produce samples of remarkable quality and diversity, and they are now the method of choice across a wide range of domains, including image synthesis~\cite{dhariwal2021diffusion,karras2022elucidating,rombach2022high,saharia2022photorealistic}, audio~\cite{kong2021diffwave} and video generation~\cite{ho2022video}, inverse problems~\cite{chung2023diffusion}, and the design of molecules and proteins~\cite{hoogeboom2022equivariant,watson2023novo,abramson2024accurate}. Beyond their empirical success, diffusion models rest on a solid mathematical foundation: in the continuous-time formulation of~\cite{song2021score}, generation amounts to simulating a stochastic or ordinary differential equation driven by a learned score, whose marginals evolve according to an explicit partial differential equation. This structure makes it possible to reason about, and to modify, the generative process in a principled way.

In many applications, however, the distribution of interest is not that of a single pretrained model. One may wish to combine several models, each trained on different data or capturing a different concept, and to sample from a distribution that is defined in terms of all of them: their product, which retains only what all models agree on; their mixture; or more elaborate combinations, such as distributions favoring the regions where one model assigns high probability and another does not. Training a new model for every such target is costly and often impractical, and a natural alternative is to reuse the pretrained models at inference time~\cite{liu2022compositional}. Yet, naive approaches that simply combine the scores of the models do not reproduce the correct intermediate distributions, and therefore yield biased samples~\cite{du2023reduce}. Recently, \cite{skreta2025feynman} showed that this bias can be corrected exactly by means of the Feynman--Kac formula~\cite{del2004feynman}: the generated particles are assigned weights that evolve along the generation, and are periodically resampled within a sequential Monte Carlo scheme~\cite{doucet2001sequential}. These Feynman--Kac correctors, however, are derived separately for specific families of targets, namely annealed distributions, products of experts, and the geometric averages underlying classifier-free guidance~\cite{ho2022classifier}.

A complementary need arises when the generated samples are required not only to be realistic, but also to optimize a downstream objective, such as the stability or binding affinity of a molecule, or the similarity of an image to a prescribed reference. This is the task of steering, or reward-guided generation~\cite{uehara2025inference}. A principled formulation consists in maximizing the expected reward, regularized by the Kullback--Leibler divergence to the pretrained distribution, whose solution is the pretrained distribution exponentially tilted by the reward. One approach is to fine-tune the pretrained model toward this tilted distribution~\cite{black2024training,domingo2025adjoint}, which, however, requires additional training for every reward. Inference-time methods, such as classifier guidance~\cite{dhariwal2021diffusion} and SMC-based guidance~\cite{wu2023practical,singhal2025general}, instead approximate the tilted distribution during generation, without modifying the model, but are typically formulated for a single pretrained model.

In this paper, we unify and generalize these two lines of work. We introduce the generalized Feynman--Kac correctors (GFKC), which enable sampling from an arbitrary target function of several pretrained diffusion models, with an optional reward-based tilt. Our main result, Theorem~\ref{thm:main_theorem}, shows that the geometric interpolation between an arbitrary initial function and an arbitrary terminal function of the models satisfies a Feynman--Kac PDE, and provides its potential in closed form. To evaluate this potential along the generation, we derive a generalization of the log-probability ODE~\cite{chen2018neural,grathwohl2019ffjord}, which tracks the density of every model along trajectories generated by a different velocity field. The resulting algorithm operates entirely in the log domain, performs the correction on the fly during the generation process, and requires neither retraining nor fine-tuning of the models. The framework of~\cite{skreta2025feynman} is recovered as a special case.

The remainder of the paper is organized as follows. Section~\ref{sec:preliminaries} reviews continuous-time diffusion models, the Feynman--Kac correctors, and reward-guided tilting. Section~\ref{sec:derivation} presents our main result, the generalized log-probability ODE, the incorporation of steering, and the resulting GFKC algorithm. Section~\ref{sec:numerical_illustrations} illustrates the method in two settings of different nature: an analytic setting of overlapping Gaussian distributions, where the generated samples are compared with exact ground-truth densities, and a high-dimensional setting of diffusion models trained on the MNIST dataset~\cite{deng2012mnist}, where we demonstrate both the sampling from non-trivial combinations of the models and the effect of steering. All numerical experiments are fully reproducible, and the code, together with the pretrained models, is publicly available at~\url{https://github.com/WaelMattar/GFKC}.

\section{Preliminaries}\label{sec:preliminaries}

We review continuous-time diffusion models (Section~\ref{subsec:diffusion}), the Feynman--Kac correctors (Section~\ref{subsec:FKC}), and reward-guided tilting (Section~\ref{subsec:tilting}), which together form the basis of our method.

\subsection{Continuous-time diffusion models}\label{subsec:diffusion}

Diffusion models~\cite{sohl2015deep,ho2020denoising,song2021score} generate data by learning to reverse a gradual noising process. In the continuous-time formulation of~\cite{song2021score}, the noising process is described by a forward stochastic differential equation (SDE)
\begin{equation}\label{eqn:forward_SDE}
  \mathrm{d}x_u = f(x_u, u)\,\mathrm{d}u + \sigma_{\mathrm{diff}}(u)\,\mathrm{d}W_u, \qquad u \in [0,1], \qquad x_0 \sim p_{\mathrm{data}},
\end{equation}
where $f$ is a drift, $\sigma_{\mathrm{diff}}$ a scalar diffusion coefficient and $W_u$ a standard Brownian motion in $\mathbb{R}^d$. The coefficients are chosen so that the marginal distribution $q_u$ of $x_u$ approaches a tractable prior as $u \rightarrow 1$. A widely used choice, which we adopt in Section~\ref{sec:numerical_illustrations}, is the variance-preserving (VP) SDE, with $f(x,u) = -\tfrac{1}{2}\beta(u)\,x$ and $\sigma_{\mathrm{diff}}(u) = \sqrt{\beta(u)}$ for a positive noise schedule $\beta$.


Generation proceeds backward in $u$. To be consistent with the convention used throughout this paper, we reparametrize time as $t = 1 - u$ and denote by $p(x,t) := q_{1-t}(x)$ the corresponding marginals, so that $p(\cdot, 0) \approx \mathcal{N}(0, I_d)$ corresponds to noise and $p(\cdot, 1) = p_{\mathrm{data}}$ to data. By the classical time-reversal result of~\cite{anderson1982reverse}, the marginals $p(\cdot, t)$ are reproduced by the family of generative dynamics
\begin{equation}\label{eqn:reverse_SDE_family}
  \mathrm{d}x_t = \Big[-f(x_t, 1-t) + \tfrac{1+\xi}{2}\,\sigma_{\mathrm{diff}}(1-t)^2\, s(x_t, t)\Big]\mathrm{d}t + \sqrt{\xi}\,\sigma_{\mathrm{diff}}(1-t)\,\mathrm{d}W_t, \qquad \xi \geq 0,
\end{equation}
where $s(x,t) = \nabla \log p(x,t)$ is the score. The case $\xi = 1$ recovers the reverse-time SDE, whereas $\xi = 0$ yields the deterministic probability flow ODE~\cite{song2021score}
\begin{equation}\label{eqn:probability_flow}
  \frac{\mathrm{d}x_t}{\mathrm{d}t} = v(x_t, t), \qquad v(x,t) = -f(x, 1-t) + \tfrac{1}{2}\,\sigma_{\mathrm{diff}}(1-t)^2\, s(x,t).
\end{equation}
Intermediate values of $\xi$ interpolate between the two regimes; see also~\cite{karras2022elucidating}. Regardless of $\xi$, the marginals satisfy the continuity equation
\begin{equation}\label{eqn:continuity}
  \partial_t p(x,t) = -\big\langle \nabla, p(x,t)\, v(x,t) \big\rangle.
\end{equation}

In practice, the score is unknown and is approximated by a neural network trained via denoising score matching~\cite{hyvarinen2005estimation,vincent2011connection}. A common parametrization, which we also adopt, is the noise-prediction objective of~\cite{ho2020denoising}: a network $\hat{\varepsilon}(x, u)$ is trained to minimize $\mathbb{E}\,\|\hat{\varepsilon}(x_u, u) - \varepsilon\|^2$, where $x_u = \alpha(u) x_0 + \sigma(u)\,\varepsilon$ with $\varepsilon \sim \mathcal{N}(0, I_d)$, and the score is recovered as $s(x,t) \approx -\hat{\varepsilon}(x,t)/\sigma(t)$. Finally, the probability flow ODE~\eqref{eqn:probability_flow} also enables likelihood evaluation: along its trajectories, the instantaneous change-of-variables formula~\cite{chen2018neural,grathwohl2019ffjord} gives
\begin{equation}\label{eqn:log_prob_ODE}
  \frac{\mathrm{d}}{\mathrm{d}t} \log p(x_t, t) = -\big\langle \nabla, v(x_t, t) \big\rangle,
\end{equation}
where, in high dimension, the divergence is typically estimated with the Hutchinson trace estimator~\cite{hutchinson1989stochastic}. As discussed in Section~\ref{sec:derivation}, equation~\eqref{eqn:log_prob_ODE} only holds when the particle is transported by the velocity field of the model whose density is evaluated, and must be generalized in our multi-agent setting. This equation will become essential in deriving our main algorithm.

\subsection{Feynman--Kac correctors}\label{subsec:FKC}

Many distributions of practical interest are not the marginals of a single trained diffusion model, even when they are explicitly defined in terms of such models. Typical examples include the annealed distribution $p^{\beta}$, the product $p_1 p_2$ of two models, and the geometric average $p_1^{1-w} p_2^{w}$, which underlies classifier-free guidance~\cite{ho2022classifier}. A natural heuristic is to simulate~\eqref{eqn:reverse_SDE_family} with the correspondingly combined scores, e.g., $\beta s$ or $s_1 + s_2$. However, the resulting dynamics do not reproduce the desired intermediate distributions, and the generated samples are therefore biased~\cite{du2023reduce}.

To address this issue, \cite{skreta2025feynman} proposed the Feynman--Kac correctors (FKC). Their key observation is that, for each of the aforementioned targets, the time derivative of the desired path of distributions can be written as a continuity equation of the form~\eqref{eqn:continuity}, augmented with a term of the form $p\,g$, where $g$ is commonly referred to as the \emph{potential function}, see~\cite{del2004feynman}. In other words, it satisfies a Feynman--Kac partial differential equation (PDE), with a velocity field and a potential that are explicitly computable from the pretrained scores. By the Feynman--Kac formula~\cite{del2004feynman}, such a PDE can be simulated by a population of weighted particles: the particles are transported by the velocity field, while their log-weights accumulate the potential along the trajectories. The weights are then used within a sequential Monte Carlo (SMC) scheme~\cite{doucet2001sequential,del2004feynman}, in which particles are periodically resampled in proportion to their weights, which prevents weight degeneracy and focuses the computational effort on the relevant regions of space. This principle is formalized in the following lemma, which will become essential to proving our main result.

\begin{lemma}\label{lem:feynman_kac}
    Let $p(x,t):\mathbb{R}^d\times [0,1]\rightarrow \mathbb{R}_+$ be a function with $p(\cdot, t) \in L^1(\mathbb{R}^d)$ for every $t \in [0,1]$. Assume that $p(x,t)$ adheres to the following Feynman--Kac PDE
    \begin{equation}\label{eqn:Feynman_Kac_diffusion_models}
        \partial_t p(x, t) = - \big\langle \nabla, p(x, t) \, v(x, t) \big\rangle + p(x, t)\, g(x, t)
    \end{equation}
    for some vector field $v(x,t)$ in $\mathbb{R}^d$, and a real-valued potential function $g(x,t)$. Then, samples from the normalized density $\widetilde{p}$ can be drawn by simulating the weighted SDE
    \begin{equation}\label{eqn:SDE_simulation}
        \mathrm{d}x_t = v(x_t, t)\,\mathrm{d}t \quad \text{and} \quad \mathrm{d}w_t = \overline{g}(x_t,t)\,\mathrm{d}t,
    \end{equation}
    where
    \begin{equation}\label{eqn:g_bar}
        \overline{g}(x,t) = g(x,t) - \int_{\mathbb{R}^d}g(x^\prime,t)\,\widetilde{p}(x^\prime,t)\,\mathrm{d}x^\prime.
    \end{equation}
\end{lemma}

We provide the proof of Lemma~\ref{lem:feynman_kac} in Appendix~\ref{app:lemma_proof} for completeness. In practice, the weights in~\eqref{eqn:SDE_simulation} relate the simulated particles to the normalized density $\widetilde{p}(\cdot, t)$ as follows~\cite{skreta2025feynman}. Denoting by $w_t = \int_0^t g(x_s, s)\,\mathrm{d}s$ the log-weight accumulated along a trajectory, one has, for any bounded test function $\phi:\mathbb{R}^d\rightarrow\mathbb{R}$,
\begin{equation*}
    \mathbb{E}_{x \sim \widetilde{p}(\cdot, t)}\big[\phi(x)\big] = \frac{1}{Z_t}\, \mathbb{E}\big[e^{w_t}\, \phi(x_t)\big], \qquad Z_t = \mathbb{E}\big[e^{w_t}\big],
\end{equation*}
where the expectations on the right-hand side are taken over the trajectories of~\eqref{eqn:SDE_simulation}. Given $K\in\mathbb{N}$ simulated particles $\{(x_t^{(k)}, w_t^{(k)})\}_{k=1}^K$, replacing both expectations by empirical averages yields the self-normalized importance sampling estimator
\begin{equation*}\label{eqn:SNIS}
  \mathbb{E}_{x \sim \widetilde{p}(\cdot, t)}\big[\phi(x)\big] \approx \sum_{k=1}^{K} \frac{e^{w_t^{(k)}}}{\sum_{\ell=1}^{K} e^{w_t^{(\ell)}}}\, \phi\big(x_t^{(k)}\big),
\end{equation*}
which converges to the exact expectation as $K \to \infty$. Accordingly, approximate samples from $\widetilde{p}(\cdot, t)$ are obtained by drawing particles from the categorical distribution with probabilities proportional to $\{e^{w_t^{(k)}}\}_{k=1}^K$, which is precisely the resampling step of Algorithm~\ref{alg:GFKC}.

Lemma~\ref{lem:feynman_kac} is a direct consequence of the Feynman--Kac formula~\cite{del2004feynman}, and underlies the constructions of~\cite{skreta2025feynman}. Note that when $g \equiv 0$, the PDE~\eqref{eqn:Feynman_Kac_diffusion_models} reduces to the continuity equation~\eqref{eqn:continuity}, the weights remain constant, and the lemma recovers standard sampling with the velocity field $v$. Moreover, the centering term in~\eqref{eqn:g_bar} is common to all particles at a given time, and therefore cancels upon normalization of the weights. Consequently, in practice one may accumulate the uncentered potential $g$ and normalize the weights only when they are used, as is done in Algorithm~\ref{alg:GFKC}. Finally, we emphasize that, in~\cite{skreta2025feynman}, the potential $g$ is derived separately for each specific family of targets. Theorem~\ref{thm:main_theorem} below removes this restriction by providing the potential for an arbitrary combination of the agents.

\subsection{Reward-guided tilting}\label{subsec:tilting}

A complementary line of work concerns the adaptation of a pretrained diffusion model to maximize a downstream reward $r : \mathbb{R}^d \to \mathbb{R}$ at inference time, without fine-tuning~\cite{uehara2025inference}. Maximizing the reward alone would collapse the generation onto its maximizers and discard the information encoded in the pretrained model. The standard remedy is to regularize the objective by the Kullback--Leibler divergence to the pretrained distribution $p_{\mathrm{pre}}$,
\begin{equation}\label{eqn:KL_regularized}
  q^{\star} = \operatorname*{arg\,max}_{q} \; \Big\{ \mathbb{E}_{x \sim q}\big[r(x)\big] - \kappa\, \mathrm{KL}\big(q \,\|\, p_{\mathrm{pre}}\big) \Big\}, \qquad \kappa > 0.
\end{equation}
This problem admits the closed-form solution
\begin{equation}\label{eqn:tilted_solution}
  q^{\star}(x) = \frac{1}{Z_{\kappa}}\, p_{\mathrm{pre}}(x)\, e^{r(x)/\kappa}, \qquad Z_{\kappa} = \int_{\mathbb{R}^d} p_{\mathrm{pre}}(x)\, e^{r(x)/\kappa}\, \mathrm{d}x,
\end{equation}
which follows from the identity
\begin{equation}\label{eqn:gibbs_identity}
  \mathbb{E}_{q}[r] - \kappa\, \mathrm{KL}(q \,\|\, p_{\mathrm{pre}}) = \kappa \log Z_{\kappa} - \kappa\, \mathrm{KL}(q \,\|\, q^{\star}),
\end{equation}
valid for any $q$ with $\mathrm{KL}(q \,\|\, p_{\mathrm{pre}}) < \infty$. Indeed, the right-hand side is maximized if and only if $q = q^{\star}$. The parameter $\kappa$ balances the two objectives: as $\kappa \to \infty$, the tilted distribution $q^{\star}$ reduces to $p_{\mathrm{pre}}$, whereas as $\kappa \to 0^+$, it concentrates on the maximizers of $r$ within the support of $p_{\mathrm{pre}}$.

As reviewed in~\cite{uehara2025inference}, existing inference-time methods, such as classifier guidance~\cite{dhariwal2021diffusion}, SMC-based guidance~\cite{wu2023practical} and value-based sampling, can be viewed as approximations of the optimal denoising process that samples from~\eqref{eqn:tilted_solution}. A central ingredient of these methods is a look-ahead function, the soft value function, which predicts, from an intermediate state $x_t$, the reward attained at the end of the generation. Since the soft value function is generally intractable, it is commonly approximated by evaluating the reward at an estimate $\widehat{x}_1$ of the terminal state, obtained for instance from the posterior mean via Tweedie's formula~\cite{chung2023diffusion} or by simulating the generative dynamics up to the terminal time. We adopt the latter approach in Section~\ref{sec:numerical_illustrations}.

In this work, we apply the same principle to targets composed of several diffusion models. Given a terminal function $h_1$ of the agents, we denote by $\widetilde{h}_1$ its normalization, and consider the tilted target~\eqref{eqn:tilted_target}, which, by~\eqref{eqn:tilted_solution}, is the solution of~\eqref{eqn:KL_regularized} with $p_{\mathrm{pre}} = \widetilde{h}_1$.

\paragraph{Our contribution.} The FKC framework~\cite{skreta2025feynman} provides exact weighted sampling for specific families of targets built from pretrained models, while reward-guided methods~\cite{uehara2025inference} tilt the distribution of a single pretrained model toward high values of a reward. In this paper, we unify and generalize both approaches. We derive the Feynman--Kac potential for an arbitrary target function $h_1$ of $N$ diffusion models, interpolated geometrically from an arbitrary initial function $h_0$, with an optional reward-based tilt. The resulting GFKC algorithm performs the correction on the fly, during the generation process, and requires neither retraining nor fine-tuning of the agents.

\section{Derivation of the generalized Feynman-Kac correctors algorithm}\label{sec:derivation}

In this section, we present our main result, which generalizes the Feynman--Kac correctors of~\cite{skreta2025feynman} from specific families of targets to arbitrary target functions of several pretrained diffusion models, and incorporates the reward-based tilting of Section~\ref{subsec:tilting} into the same framework.

We consider $N \geq 1$ pretrained diffusion models, henceforth referred to as \emph{agents}, with marginals $p_1(x,t), \dots, p_N(x,t)$, scores $s_i = \nabla \log p_i$ and velocity fields $v_i$, $i = 1, \dots, N$. As in Section~\ref{subsec:diffusion}, each agent is assumed to satisfy the continuity equation~\eqref{eqn:continuity} with its own velocity field, namely $\partial_t p_i = -\langle \nabla, p_i v_i \rangle$. Moreover, we assume that all agents share a common, tractable initial distribution at $t = 0$, as is the case for diffusion models with the same forward process, $p_i(\cdot, 0) = \mathcal{N}(0, I_d)$.

Our goal is to sample from a target distribution that is defined pointwise in terms of the agents. Specifically, given a positive and differentiable \emph{terminal function} $h_1 : \mathbb{R}_+^N \to \mathbb{R}_+$, we wish to sample from the normalized density proportional to
\begin{equation*}
    h_1\big(p_1(x,1), \dots, p_N(x,1)\big)
\end{equation*}
where the agents are evaluated at the terminal time $t = 1$, i.e., at their data distributions. This formulation encompasses, for instance, products of experts, geometric averages underlying classifier-free guidance~\cite{ho2022classifier}, and annealed distributions, which are the targets considered in~\cite{skreta2025feynman}. It also includes targets that are not of product form, such as the mixtures and ratios studied in Section~\ref{sec:numerical_illustrations}. We refer to a terminal function as \emph{admissible} whenever the resulting density is both differentiable and integrable, so that its normalization is well defined.

Since the target is generally not the terminal marginal of any single agent, we construct a path of distributions connecting a tractable initial distribution to the target. To this end, we introduce an \emph{initial function} $h_0 : \mathbb{R}_+^N \to \mathbb{R}_+$, chosen so that $h_0\big(p_1(\cdot,0), \dots, p_N(\cdot,0)\big)$ is easy to sample from. For instance, when the agents share a common prior $p_0$, any $h_0$ with $h_0(p_0, \dots, p_0) = p_0$, such as the geometric mean, is a valid choice. A natural and flexible family of initial functions is given by
\begin{equation*}\label{eqn:quasi_arithmetic}
  h_0(p_1, \dots, p_N) = \phi^{-1}\Big( \frac{1}{N} \sum_{i=1}^{N} \phi(p_i) \Big),
\end{equation*}
where $\phi : \mathbb{R}_+ \rightarrow \mathbb{R}$ is continuous and strictly monotone. This family includes the arithmetic mean for $\phi(z) = z$, the geometric mean for $\phi(z) = \log z$, and, more generally, the power means for $\phi(z) = z^{\alpha}$ with $\alpha \neq 0$, among which the harmonic mean corresponds to $\alpha = -1$. Every member of the family satisfies $h_0(p_0, \dots, p_0) = p_0$, and is therefore a valid initial function whenever the agents share a common prior $p_0$. The choice of $\phi$ thus offers considerable freedom in designing the initial distribution of the path~\eqref{eqn:trajectory_distribution}, without affecting the target $h_1$ at $t = 1$. In our experiments in Section~\ref{sec:numerical_illustrations}, we use the geometric mean.

The two functions are then interpolated geometrically by means of differentiable schedules $a_0, a_1 : [0,1] \to [0,1]$ satisfying $a_0(0) = a_1(1) = 1$ and $a_0(1) = a_1(0) = 0$. Throughout, we write $\partial_{\log p_j} \log h := p_j\, \partial_{p_j} \log h$ for the derivative of $\log h$ with respect to $\log p_j$, which is the natural quantity in the log domain.

The following theorem shows that the resulting path satisfies a Feynman--Kac PDE of the form~\eqref{eqn:Feynman_Kac_diffusion_models}, with an explicit potential, so that Lemma~\ref{lem:feynman_kac} applies.

\begin{theorem}~\label{thm:main_theorem}
    Let $\{p_i(x,t)\}_{i=1}^{N}$ be a set of diffusion models with corresponding vector fields $v_i$. Let $h_0:\mathbb{R}^N\rightarrow\mathbb{R}$ be an initial annealing function, and denote by $h_1:\mathbb{R}^N\rightarrow\mathbb{R}$ the target distribution function. Then, simulating~\eqref{eqn:SDE_simulation} with an arbitrary unifying vector field $v(x,t)=V\big(v_1(x,t),\dots, v_N(x,t)\big)$ and
    \begin{align}~\label{eqn:g_weights_evolution}
        g & = \dot{a}_0 \log h_0 + \dot{a}_1 \log h_1 + \langle \nabla, v \rangle \nonumber \\ & - \sum_{j=1}^N \big( \langle \nabla, v_j \rangle + \langle v_j - v, s_j \rangle \big)\big(a_0 \partial_{\log p_j} \log h_0 + a_1 \partial_{\log p_j} \log h_1\big),
    \end{align}
    amounts to sampling from the trajectory 
    \begin{equation}~\label{eqn:trajectory_distribution}
        p(x,t) = h_0\big(p_1(x,t),\dots,p_N(x,t)\big)^{a_0(t)}h_1\big(p_1(x,t),\dots,p_N(x,t)\big)^{a_1(t)},
    \end{equation}
    which geometrically interpolates between $h_0$ at $t=0$ and $h_1$ at $t=1$.
\end{theorem}

\begin{proof}
    As in Equation~\eqref{eqn:g_weights_evolution}, we omit the variables $x$ and $t$ throughout the proof for compactness. In the logarithmic domain, the trajectory distribution $p$ of~\eqref{eqn:trajectory_distribution} becomes
    \begin{equation*}
        \log p = a_0\log h_0\big(p_1,\dots,p_N\big) + a_1 \log h_1 \big(p_1,\dots,p_N\big).
    \end{equation*}
    Consequently, differentiating $\log p$ with respect to $t$ and multiplying the result with $p$ we get $\partial_t p$:
    \begin{equation}~\label{eqn:diff_log_domain}
        \partial_t p = p \left[
        \dot{a}_0 \log h_0 + a_0 \sum_{i=1}^N (\partial_{p_i} \log h_0) \, (\partial_t p_i) + \dot{a}_1 \log h_1 + a_1 \sum_{i=1}^N (\partial_{p_i} \log h_1) \, (\partial_t p_i) \right] =: p\mathcal{C}.
    \end{equation}
    We now simplify the term $\partial_t p_i$ in $\mathcal{C}$ using the agents' continuity equations~\eqref{eqn:continuity} and get
    \begin{align*}
        \partial_t p_i & = -\langle \nabla, p_iv_i \rangle \\
        & = - p_i \langle \nabla, v_i \rangle - \langle v_i, \nabla p_i \rangle \\
        & = - p_i \big( \langle \nabla, v_i \rangle + \langle v_i, s_i \rangle \big).
    \end{align*}
    Here we used the fact that $\nabla p_i = p_i\nabla \log p_i = p_is_i$ for all $i=1,\dots, N$. Before proceeding to the final step of the proof, we note that for an arbitrary unifying vector field $v$, the term $\langle \nabla, pv\rangle$ can be rewritten as
    \begin{equation*}
        \langle \nabla, pv\rangle = p\langle \nabla, v\rangle + \langle v, \nabla p\rangle = p \big(\langle \nabla, v\rangle + \langle v, \nabla \log p\rangle\big) =: p \mathcal{D},
    \end{equation*}
   while the term $\mathcal{D}$ can be simplified by directly computing $\nabla \log p$:
    \begin{equation*}
        \nabla \log p = \nabla \big( a_0 \log h_0 + a_1 \log h_1 \big) = \sum_{i=1}^N \big(a_0 \partial_{p_i} \log h_0 + a_1 \partial_{p_i} \log h_1\big) p_i s_i,
    \end{equation*}
    which overall gives
    \begin{equation}~\label{eqn:D_function}
        \mathcal{D} = \langle \nabla, v \rangle + \sum_{i=1}^N p_i \langle v, s_i \rangle \big(a_0 \partial_{p_i} \log h_0 + a_1 \partial_{p_i} \log h_1\big).
    \end{equation}

    Finally, combining~\eqref{eqn:diff_log_domain} with~\eqref{eqn:D_function} we get
    \begin{align*}
        \partial_t p & = p\mathcal{C} - \langle \nabla, p v \rangle + \langle \nabla, p v \rangle \\ & = -\langle \nabla, p v \rangle + (\mathcal{C} + \mathcal{D}) p \\ & = -\langle \nabla, p v \rangle + gp,
    \end{align*}
    which coincides precisely with the Feynman-Kac equation~\eqref{eqn:Feynman_Kac_diffusion_models} with the potential function $g$ as in~\eqref{eqn:g_weights_evolution}.
\end{proof}

Theorem~\ref{thm:main_theorem} establishes an interpolation between an initial function $h_0$ and a desired terminal function $h_1$ along the path given in~\eqref{eqn:trajectory_distribution}. The sampling process consists of simulating the SDE~\eqref{eqn:SDE_simulation} with a suitable choice of the function $g$, which is constructed so that, at the terminal time $t=1$, the reweighted samples are distributed according to $h_1$. The choice of $V$ is discussed before the derivation of our main algorithm.

To evaluate the density of a sample under a flow model, one typically relies on the instantaneous change-of-variables formula, commonly known as the log-probability ODE. In our setting, however, this ODE cannot be applied directly, for reasons we now explain, and must instead be suitably modified.

Although the particles evolve under a velocity field $v(x,t)$, we must compute, for every particle, the densities $p_j(x,t)$ associated with each model $j$. In other words, even though a particle is transported by a velocity field $v\neq v_j$, its density must still be evaluated under every model $j$. This is achieved by a generalization of the log-probability ODE~\eqref{eqn:log_prob_ODE} as follows.
\begin{align*}
    \frac{\mathrm{d}}{\mathrm{d}t} \log p_j(x_t, t)
    & = \frac{1}{p_j(x_t, t)} \frac{\mathrm{d}}{\mathrm{d}t} p_j(x_t, t) \\
    & = \frac{1}{p_j(x_t, t)} \left(
    \frac{\partial}{\partial t} p_j(x_t, t) + \left\langle \nabla p_j(x_t, t), \frac{\mathrm{d}x_t}{\mathrm{d}t} \right\rangle
    \right) \\
    & = \frac{1}{p_j(x_t, t)} \left( - \big\langle \nabla, p_j(x_t, t) v_j(x_t, t) \big\rangle
     + \left\langle \nabla p_j(x_t, t), \frac{\mathrm{d}x_t}{\mathrm{d}t} \right\rangle
    \right) \\
    & = p_j^{-1} \big(-p_j \langle \nabla, v_j \rangle - \langle v_j, \nabla p_j \rangle + \langle v, \nabla p_j \rangle \big) \\
    & = -\langle \nabla, v_j(x_t, t) \rangle - \langle v_j(x_t, t) - v(x_t, t), s_j(x_t, t) \rangle.
\end{align*}
Here, the second line follows by expanding the total time derivative into the partial time derivative and the contribution from the time derivative of the trajectory $x_t$; the third line uses the continuity equation for $p_j$; the fourth line uses $\mathrm{d}x_t/\mathrm{d}t = v(x_t,t)$, since the trajectory is generated by the velocity field $v$; and the fifth line uses the identity $\nabla p_j = p_j\, s_j$. Note that when $v = v_j$, the second term vanishes and we recover the standard log-probability ODE as one would expect.

We now examine how the inclusion of steering affects the framework. Suppose now that we wish to steer the generative process toward the tilted target
\begin{equation}~\label{eqn:tilted_target}
    h_1^\kappa\big(p_1(x,t),\dots,p_N(x,t)\big) \propto h_1\big(p_1(x,t),\dots,p_N(x,t)\big) e^{r(x)/\kappa}
\end{equation}
which is the equivalent of~\eqref{eqn:tilted_solution}, solving the optimization problem in~\eqref{eqn:KL_regularized}, where $\kappa>0$ is the steering parameter, and $r(x)$ is a reward function to be maximized. The only modification required in the algorithm is the addition of the term $\dot{a}_1(t)r(x)/\kappa$ to the function $g$ in~\eqref{eqn:g_weights_evolution}. This additional term alters the evolution of the distribution and tilts the path $p$ toward the maxima of the reward function $r$.

The steering parameter $\kappa$ controls the strength of the bias toward maximizing $r$ throughout the generation process. In the limit $\kappa\rightarrow\infty$, the additional term becomes negligible, and the original path of distributions is recovered. In contrast, as $\kappa\rightarrow 0^+$, the steering term becomes increasingly dominant and ultimately governs the generative process, as illustrated in Figure~\ref{fig:Gaussian_steering_effect}.

Building on Theorem~\ref{thm:main_theorem} and the mathematical evaluations presented in this section, we derive a practical algorithm, henceforth referred to as \emph{the GFKC algorithm}, for sampling from any admissible target function $h_1$, with optional reward-based steering. For clarity of presentation, the algorithm refers to the update equations~\eqref{eqn:A}--\eqref{eqn:G}, which are stated below.

Although Theorem~\ref{thm:main_theorem} allows for an arbitrary unifying vector field $v = V(v_1, \dots, v_N)$, Algorithm~\ref{alg:GFKC} relies on a particular choice of $V$. The particles are divided into $N$ equally sized sub-populations, and for the $i$-th sub-population we take $V$ to be the projection onto its $i$-th argument, $V(v_1, \dots, v_N) = v_i$ for $i = 1, \dots, N$,
so that each particle is transported by the velocity field of its own agent. Theorem~\ref{thm:main_theorem} is thus applied with $v = v_i$, which yields the potential $g_i$ in~\eqref{eqn:E}. In this way, each sub-population is driven by a genuine generative field and explores the high-density regions of its agent, while the terms $u_{ij}$ in~\eqref{eqn:C} correct for the discrepancy with the other agents. Since the path~\eqref{eqn:trajectory_distribution} and its normalizing constant do not depend on the choice of $V$, all sub-populations target the same normalized density, and their particles can be pooled into a single weighted population for estimation and resampling. After resampling, each particle retains the label of its agent.

The quantities appearing in Algorithm~\ref{alg:GFKC} are defined as follows.

\begin{align}
& q(x_t, t) = \dot{a}_0(t) \log h_0(\log p_1(x_t, t), \dots, \log p_n(x_t, t)) \notag \\
& \hspace{35.0mm} + \dot{a}_1(t) \log h_1(\log p_1(x_t, t), \dots, \log p_n(x_t, t)) \tag{A} \label{eqn:A}\\
& b_j(x_t, t) = a_0(t) \, \partial_{\log p_j} \log h_0(\log p_1(x_t, t), \dots, \log p_n(x_t, t)) \notag \\
& \hspace{35.0mm} + a_1(t) \, \partial_{\log p_j} \log h_1(\log p_1(x_t, t), \dots, \log p_n(x_t, t)) \tag{B} \label{eqn:B}\\
& u_{ij}(x_t, t) = \langle \nabla, v_j(x_t, t) \rangle + \langle v_j(x_t, t) - v_i(x_t, t), s_j(x_t, t) \rangle \tag{C} \label{eqn:C}\\
& \frac{d}{dt} \log p_j(x_t, t) = -u_{ij}(x_t, t) \tag{D} \label{eqn:D}\\
& g_i(x_t, t) = \langle \nabla, v_i(x_t, t) \rangle + q(x_t, t) - \sum_{j=1}^N   u_{ij}(x_t, t) b_j(x_t, t) \tag{E} \label{eqn:E}\\
& dw_t = g_i(x_t, t) dt \tag{F} \label{eqn:F}\\
& dx_t = v_i(x_t, t)dt \tag{G} \label{eqn:G}
\end{align}

We emphasize that Equations~\eqref{eqn:A}--\eqref{eqn:G} are expressed entirely in terms of the log-densities $\log p_j$ and of derivatives with respect to $\log p_j$, rather than the densities $p_j$ themselves. Accordingly, all computations in the GFKC algorithm are carried out in the log domain. This is essential in practice, since in high-dimensional spaces the densities $p_j$ may take extremely large or extremely small values, which cannot be represented in floating-point arithmetic. Moreover, in some practical scenarios, the log-densities may additionally be clipped to prevent numerical overflow, as we illustrate in Section~\ref{sec:numerical_illustrations}.

\begin{algorithm}[t!]
\caption{The generalized Feynman-Kac correctors (GFKC) algorithm}\label{alg:GFKC}
\KwIn{Number of agents $N$; velocity fields $v_1,\dots,v_N$ and scores $s_1,\dots,s_N$; initial and target functions $h_0, h_1$; schedules $a_0, a_1$; reward function $r$ and steering parameter $\kappa$; total number of particles $MN$; space dimension $d$; time step $\Delta t$; set of resampling times $\mathcal{T}$}
\vspace{\baselineskip}
\KwOut{Weighted particles $\{(x_1, w_1)\}$ approximating the target $h_1e^{r/\kappa}$}
\vspace{\baselineskip}
\KwInit{\\
Assign an equal number of particles $M\in\mathbb{N}$ to each agent $i=1,\dots,N$ \\
$t \gets 0$}
\ForEach{particle}{
  Sample $x_0 \sim \mathcal{N}(0, I_d)$\;
  \For{$j = 1$ \KwTo $N$}{
    $\log p_j(x_0, 0) \gets -\frac{1}{2}\left(\|x_0\|^2 + d \log(2\pi)\right)$\;
  }
  $w_0 \gets 0$\;
}
\vspace{\baselineskip}
\KwEvol{\\
\While{$t < 1$}{
  \For{$i = 1$ \KwTo $N$}{
    \ForEach{particle $x_t$ belonging to agent $i$}{
      Compute $q(x_t, t)$ according to~\eqref{eqn:A}\;
      \If{steering is True}{
      Estimate clean samples $\widehat{x}_1$ from the intermediate states $x_t$\;
      $q(x_t, t) \gets q(x_t, t) + \dot{a}_1 r(\widehat{x}_1)/\kappa$
    }
      \For{$j = 1$ \KwTo $N$}{
        Compute $b_j(x_t, t)$ according to~\eqref{eqn:B}\;
        Compute $u_{ij}(x_t, t)$ according to~\eqref{eqn:C}\;
        $\log p_j(x_{t+\Delta t}, t+\Delta t) \gets \log p_j(x_t, t) - \Delta t \, u_{ij}(x_t, t)$ {Eq.~\eqref{eqn:D}}
      }
      Compute $g_i(x_t, t)$ according to~\eqref{eqn:E}\;
      $w_{t+\Delta t} \gets w_t + \Delta t \, g_i(x_t, t)$ {Eq.~\eqref{eqn:F}} \\
      $x_{t+\Delta t} \gets x_t + \Delta t \, v_i(x_t, t)$ {Eq.~\eqref{eqn:G}}
    }
  }
  $t \gets t + \Delta t$\;
  \If{$t \in \mathcal{T}$}{
    Resample the particles from the multinomial distribution with weights proportional to $\{\exp(w_t)\}$\;
  }
}
}
\end{algorithm}

\section{Numerical illustrations}~\label{sec:numerical_illustrations}

We illustrate the performance of Algorithm~\ref{alg:GFKC} in two settings of inherently different nature: the first involves analytic generative models for two overlapping bivariate Gaussian distributions, while the second consists of diffusion models trained on the MNIST dataset~\cite{deng2012mnist}.

\subsection{Overlapping bivariate normal distributions}

We consider a two-dimensional synthetic setting in which all quantities of interest are available in closed form, which allows the output of the GFKC algorithm~\ref{alg:GFKC} to be compared against exact ground-truth densities. The setting consists of $N = 2$ agents, each given by a diffusion model whose data distribution is an isotropic bivariate Gaussian with unit covariance,
\begin{equation}~\label{eqn:Gaussian_analytic_distributions}
    p_1 = \mathcal{N}(\mu_1, I_2) \quad \text{and} \quad p_2 = \mathcal{N}(\mu_2, I_2),
\end{equation}
with means $\mu_1 = (3, 3)$ and $\mu_2 = (3.5, 3.5)$. Since $\|\mu_1 - \mu_2\| = \sqrt{2}/2 \approx 0.71$ is smaller than the unit standard deviation of each component, the two distributions overlap substantially, and their high-density regions are difficult to separate by sampling from either model alone.

Recall that we adopt the convention that $t = 0$ corresponds to noise and $t = 1$ to data. Each agent $j\in\{1, 2\}$ defines a path of marginals $p_j(\cdot, t) = \mathcal{N}(t \mu_j, I_2)$, $t \in [0, 1]$, which interpolates between the standard Gaussian at $t = 0$ and the data distribution $p_j$ at $t = 1$. In particular, all agents share the same initial distribution $\mathcal{N}(0, I_2)$, and the scores $s_j(x, t) = -(x - t\mu_j)$ are known exactly. Consequently, no learned approximation enters the experiment, and any discrepancy between the generated and target distributions can be attributed to the sampling procedure itself rather than to model error.

The interpolation between the initial and terminal functions follows the path~\eqref{eqn:trajectory_distribution} with linear schedules $a_0(t) = 1 - t$ and $a_1(t) = t$. As the initial function we take the geometric mean of the two models
\begin{equation*}
    h_0(p_1, p_2) = \sqrt{p_1\, p_2},
\end{equation*}
which at $t = 0$ coincides with the common standard Gaussian prior. This choice ensures that the interpolation starts from the distribution from which the particles are initialized.

For the terminal function $h_1$ we consider two cases; the first one is designed to emphasize the region in which one model disagrees with the other, while the second one is its symmetrized counterpart. In particular, the first is the asymmetric ratio
\begin{equation}~\label{eqn:ratio_target}
    h_1(p_1, p_2) = \frac{p_1}{1 + C\, p_2},
\end{equation}
for some $C>0$, which favors regions of high density under $p_1$ and low density under $p_2$. The second one is
\begin{equation}~\label{eqn:symmetric_difference_target}
    h_1(p_1, p_2) = \frac{1}{2}\left( \frac{p_1}{1 + C\, p_2} + \frac{p_2}{1 + C\, p_1} \right),
\end{equation}
which favors regions where exactly one of the two models assigns high density. In both cases we set $C = 100$, so that the denominators strongly penalize regions in which the competing model is non-negligible. This choice of $C$ produces a clear yet moderate shift of mass. Nevertheless, the results do not depend on this particular choice. Owing to the substantial overlap between $p_1$ and $p_2$, the resulting targets are concentrated on relatively small regions of the plane, which makes them challenging to sample from directly.

We consider $100{,}000$ particles per agent so that the generated samples cover the whole support of the target density, which allows an accurate comparison with the ground truth across the entire domain. We emphasize that this large population is needed only for visualizing and evaluating the full density: drawing individual samples from the target requires far fewer particles, as in the MNIST experiments of Section~\ref{subsec:mnist}.

Since the setting is low-dimensional and the models are available in closed form, we compare the generated samples directly with the ground-truth target distributions. We first present the results for the two choices of $h_1$, then report the convergence behavior of the algorithm, and finally illustrate the effect of steering.

\begin{figure}[htbp]
  \centering
  \begin{subfigure}{\textwidth}
    \centering
    \includegraphics[width=\linewidth]{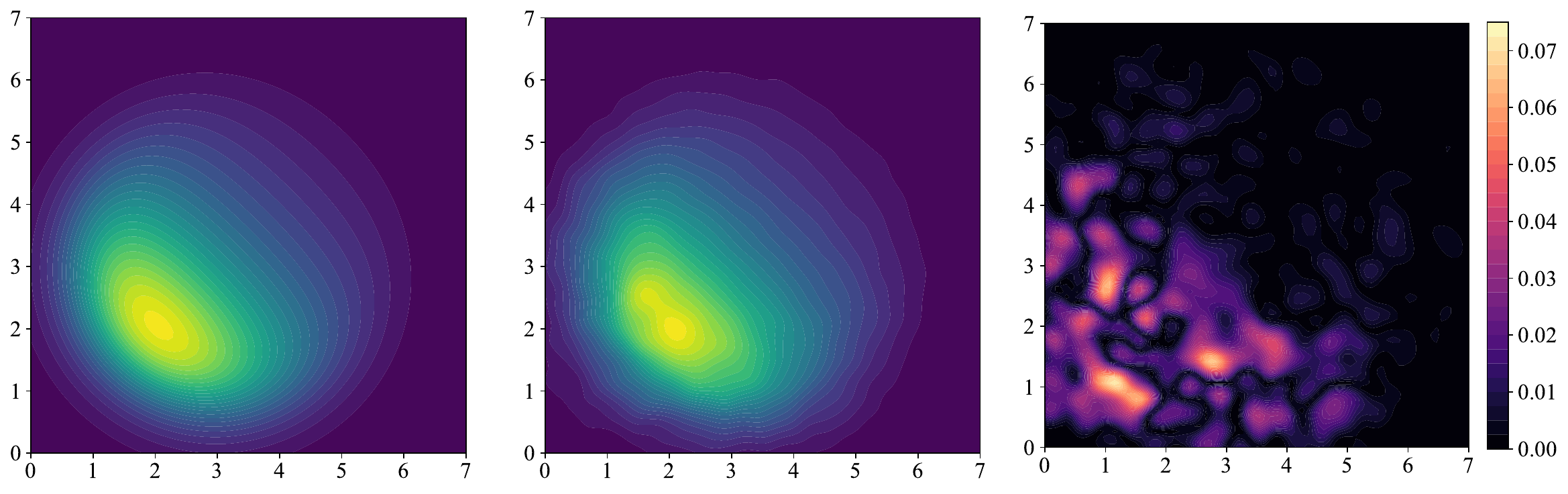}
    \label{fig:Gaussian_ratio}
  \end{subfigure}
  \vfill
  \begin{subfigure}{\textwidth}
    \centering
    \includegraphics[width=\linewidth]{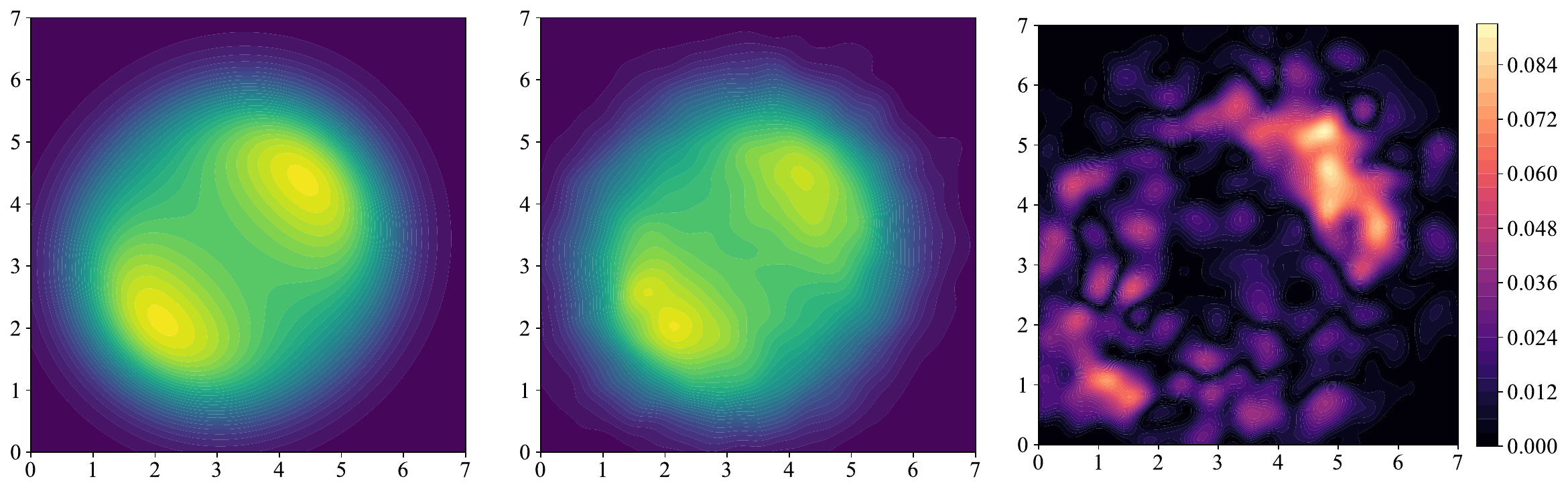}
    \label{fig:Gaussian_symmetric}
  \end{subfigure}
  \caption{Sampling with the GFKC algorithm in the bivariate Gaussian setting. Each row corresponds to one choice of the terminal function $h_1$: the asymmetric ratio of~\eqref{eqn:ratio_target} on top and its symmetrized counterpart~\eqref{eqn:symmetric_difference_target} on bottom. Left column: the ground-truth target densities, obtained by applying $h_1$ directly to the analytic densities~\eqref{eqn:Gaussian_analytic_distributions}. Middle column: the empirical densities generated by the GFKC algorithm, using $100{,}000$ particles per agent and simulating up to time $t = 1$. Right column: the relative pointwise error $|h_1 - \hat{h}_1| / \max_x h_1$ between the ground-truth and the generated densities. Note the close agreement between the left and middle columns.}
  \label{fig:Gaussians_GFKC}
\end{figure}

Figure~\ref{fig:Gaussian_convergence} reports the convergence of the GFKC algorithm for the two experiments of Figure~\ref{fig:Gaussians_GFKC} as the number of particles per agent increases. The error is measured in two ways: the Kullback--Leibler divergence $\mathrm{KL}(h_1 \,\|\, \hat{h}_1)$ between the ground-truth target $h_1$ and its empirical approximation $\hat{h}_1$ produced by the algorithm, and the $\ell_2$ distance between the two.

\begin{figure}[htbp]
    \centering
    \includegraphics[width=0.7\linewidth]{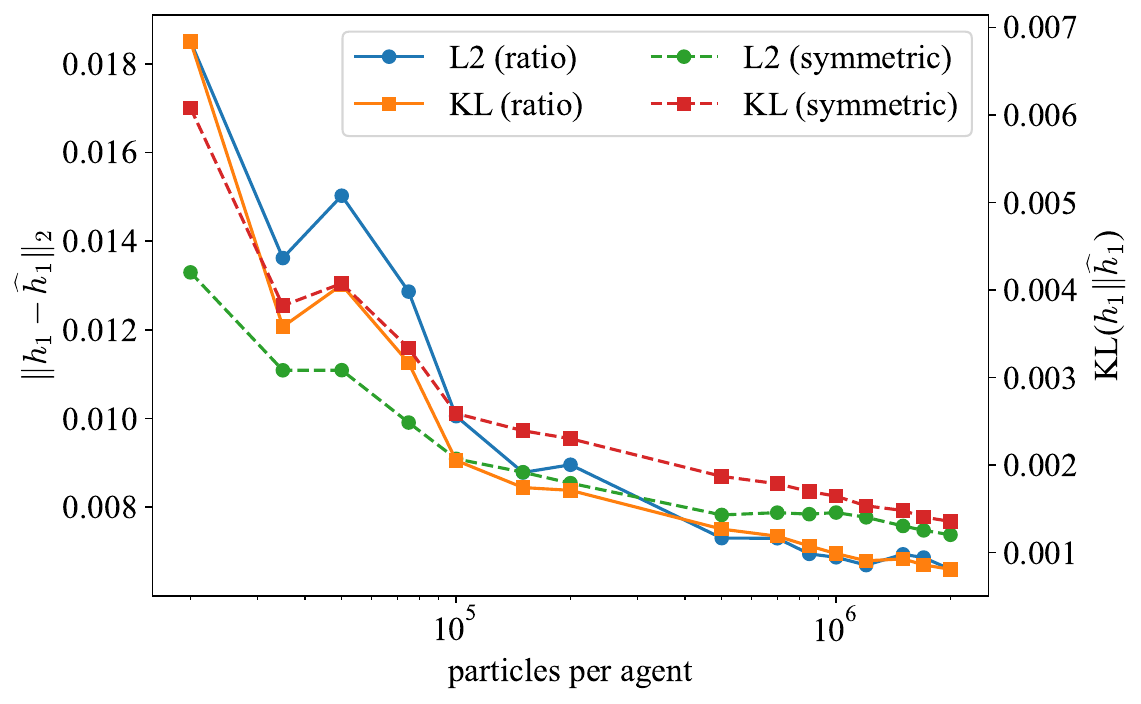}
    \caption{Convergence of the GFKC algorithm as a function of the number of particles per agent, measured by $\mathrm{KL}(h_1 \,\|\, \hat{h}_1)$ and the $\ell_2$ distance, for both choices of $h_1$.}
    \label{fig:Gaussian_convergence}
\end{figure}

Finally, for the steered variant of the algorithm, we use the reward function $r(x) = -\|x\|_2$, so that the steered target $h_1\, e^{r/\kappa}$ of~\eqref{eqn:tilted_target} tilts the terminal distribution toward points closer to the origin, with the strength of this tilt controlled by the steering parameter $\kappa > 0$.

To examine the steering effect, we consider the symmetric terminal function $h_1$ defined in~\eqref{eqn:symmetric_difference_target}, together with the prescribed reward function and steering parameter $\kappa = 1$. This choice creates a competition between two objectives: sampling from the symmetric target, which is bimodal as shown in Figure~\ref{fig:Gaussians_GFKC}, and favoring samples that lie close to the origin.

Since the mode associated with the first agent, centered around $\mu_1 = (3,3)$, is closer to the origin, the steered target naturally assigns more mass to this mode at the expense of the other. Figure~\ref{fig:Gaussian_steering_kappa_1} illustrates the effect of steering for $\kappa = 1$ and compares the result with the unsteered target $h_1$. Figure~\ref{fig:Gaussian_steering_effect} then shows the steered distributions for several values of $\kappa$, confirming that the steering effect strengthens as $\kappa$ decreases, in agreement with the theory.

\begin{figure}[htbp]
    \centering
    \includegraphics[width=1\linewidth]{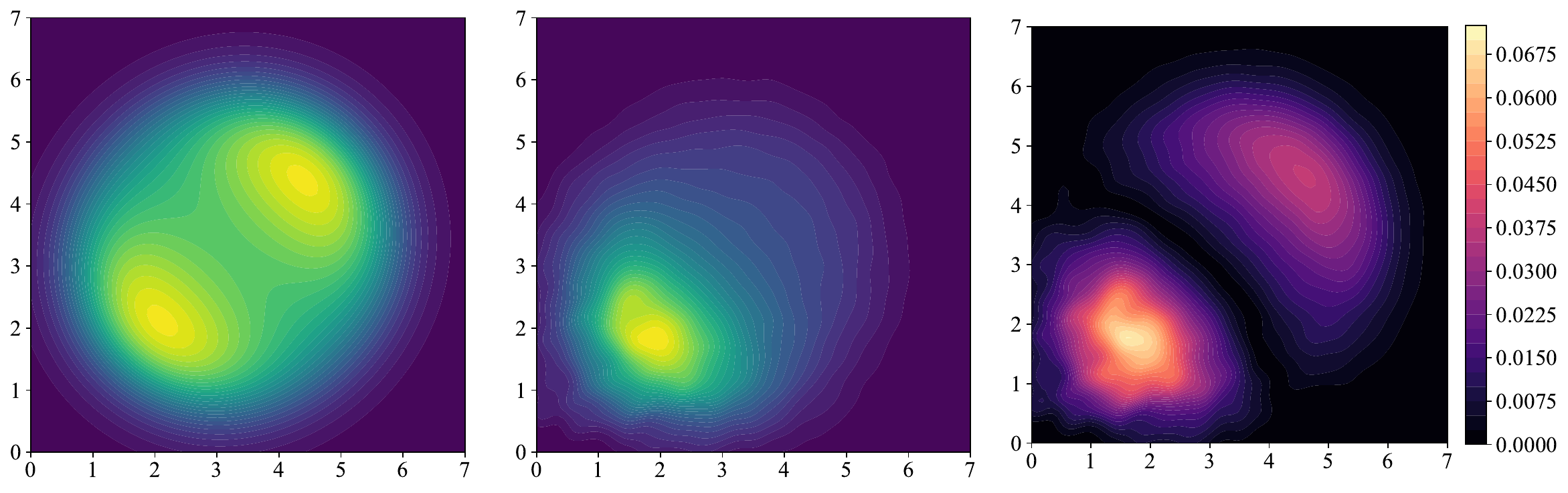}
    \caption{Effect of steering on the symmetric target $h_1$ of~\eqref{eqn:symmetric_difference_target}, with reward $r(x) = -\|x\|$ and $\kappa = 1$. Left: the analytic unsteered target density $h_1$. Middle: the empirical density generated by the GFKC algorithm with steering, using $100{,}000$ particles per agent and simulating up to time $t = 1$. Right: the pointwise difference between the steered empirical density and the unsteered target. Note the shift of mass toward the mode closer to the origin.}
    \label{fig:Gaussian_steering_kappa_1}
\end{figure}

During the steered GFKC algorithm, the reward $r$ must be evaluated at the terminal state of each particle, which is not available at intermediate times. We therefore estimate the terminal sample $\widehat{x}_1$ from the current state $x_t$ as follows: starting from $x_t$, we continue the denoising process up to time $t = 1$ using the diffusion model of the agent to which the particle belongs, and take the resulting endpoint as $\widehat{x}_1$. This projection is performed independently of the algorithm, in the sense that the particle itself, its log-densities $\log p_j$ and its weight are left unchanged. The reward $r(\widehat{x}_1)$ is then incorporated into~\eqref{eqn:A} at the equispaced steering times $t_k = 0.02\, k$ for $k = 0, \dots, 50$, which form a dense grid of the time interval and are evenly distributed across the generation process.

\begin{figure}[htbp]
    \centering
    \includegraphics[width=1\linewidth]{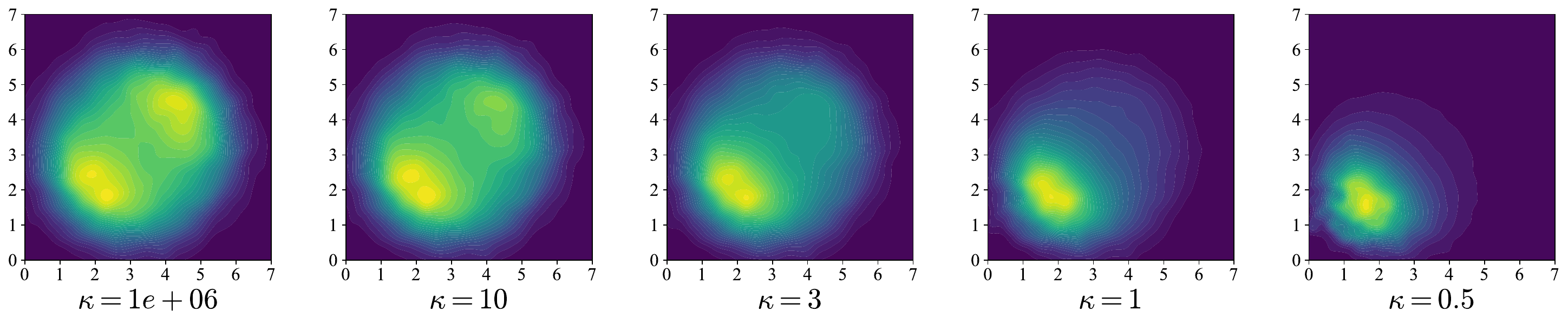}
    \caption{Effect of the steering parameter $\kappa$ on the symmetric target $h_1$ of~\eqref{eqn:symmetric_difference_target}, with reward $r(x) = -\|x\|$. Each panel shows the density generated by the GFKC algorithm for a different value of $\kappa$. As $\kappa$ decreases, the steering effect becomes increasingly pronounced, and mass is progressively transferred toward the mode closer to the origin. Conversely, for very large $\kappa$, here $\kappa = 10^6$, the steering term is negligible and the generated density coincides with the unsteered target $h_1$, as expected.}
    \label{fig:Gaussian_steering_effect}
\end{figure}

\subsection{Experimenting with the MNIST dataset}\label{subsec:mnist}

We now move from the analytic Gaussian setting to a high-dimensional setting in which the agents are learned diffusion models. Whereas in the Gaussian experiments all quantities were available in closed form, here each agent is a diffusion model trained on a subset of the MNIST dataset~\cite{deng2012mnist}, the data space has dimension $d = 28 \times 28 = 784$, and neither the scores nor the densities $p_j$ for $j=1,2$ are known analytically. The two agents are trained on two sets of digit classes that share exactly one class. The shared class thus plays the role of the region where the agents agree, while the two remaining classes, each seen by only one agent, correspond to the regions in which they disagree.

To make the distinction between the shared and the agent-specific classes as clear as possible, we select the classes so that the three digits involved are pairwise far apart. To this end, we view each image as a probability measure on the uniform $28 \times 28$ grid in $[0,1]^2$, obtained by normalizing its pixel intensities, and compare images through the Wasserstein-1 distance $W_1$ with Euclidean ground cost, computed exactly by solving the corresponding optimal transport problem. Distances between classes are measured in two complementary ways: as the $W_1$ distance between the class-mean images, and as the average $W_1$ distance over $1{,}000$ randomly drawn pairs of images from the two classes. Among all triples of digits $\{a, b, c\}$, we then select the one maximizing the smallest of the three pairwise distances, and assign to the shared class $b$ the digit whose distances to the other two are largest. This procedure leads to the choice
\begin{equation*}
    \mathcal{S}_1 = \{1, 6\}, \qquad \mathcal{S}_2 = \{1, 7\}
\end{equation*}
so that the digit $1$ is shared by both agents, while the digits $6$ and $7$ are seen exclusively by the first and second agent, respectively.

Each agent is trained on the images of its two classes from the MNIST training and test sets, with pixel intensities rescaled to $[-1, 1]$ and images flattened into vectors in $\mathbb{R}^{784}$. Since the classes of MNIST are not equally represented, each training batch is drawn with equal probability from the two classes of the agent, so that the resulting model assigns comparable mass to both digits. As in the Gaussian setting, we adopt the convention that $t = 0$ corresponds to noise and $t = 1$ to data. The forward noising process is the variance-preserving diffusion with a linear noise schedule, $\beta(u) = \beta_0 + (\beta_1 - \beta_0)\, u$ with $\beta_0 = 0.1$ and $\beta_1 = 20$, where $u = 1 - t$ denotes the forward time. The corresponding marginals are $x_u = \alpha(u)\, x_0 + \sigma(u)\, \varepsilon$, with $\varepsilon \sim \mathcal{N}(0, I_d)$, $\log \alpha(u) = -\tfrac{1}{2}\beta_0 u - \tfrac{1}{4}(\beta_1 - \beta_0) u^2$ and $\sigma(u) = \sqrt{1 - \alpha(u)^2}$. Each network is trained with the standard noise-prediction objective, from which the score is recovered as $s_j(x, t) = -\hat{\varepsilon}_j(x, t)/\sigma(t)$.

Both agents share the same architecture, a fully connected network acting directly on the flattened images. The diffusion time is encoded through a $64$-dimensional sinusoidal embedding, followed by a linear layer and a SiLU activation. The embedding is concatenated with the input image and passed through three hidden layers of width $2{,}048$ with SiLU activations, followed by a linear output layer of dimension $784$. Each agent is trained for $300{,}000$ iterations using the Adam optimizer with learning rate $2 \times 10^{-4}$ and batch size $512$, with the diffusion time sampled uniformly from $[10^{-3}, 1 - 10^{-3}]$.

Unlike in the Gaussian setting, the target $h_1$ cannot be evaluated in closed form. The results in this section are therefore assessed qualitatively, through the generated samples and their weights, rather than by comparison with a ground-truth density. Several adjustments to the GFKC algorithm are required in this setting:
\begin{enumerate}
  \item \textbf{Divergence estimation.} In the Gaussian experiments, the divergences $\langle \nabla, v_j \rangle$ appearing in~\eqref{eqn:C} and~\eqref{eqn:E} were computed exactly. In dimension $d = 784$, this is prohibitively expensive, as it requires $d$ backward passes per agent and per time step. We therefore replace it with the Hutchinson trace estimator~\cite{hutchinson1989stochastic}
  \begin{equation*}
      \langle \nabla, v_j \rangle \approx \frac{1}{L} \sum_{\ell=1}^{L} \eta_\ell^\top \nabla v_j\, \eta_\ell,
  \end{equation*}
  where $\nabla v_j$ is the Jacobian matrix of $v_j$ with respect to $x$, and $L = 10$ Gaussian probe vectors $\eta_\ell \sim \mathcal{N}(0, I_d)$.

  \item \textbf{Stochastic sampler.} The particles are evolved with a stochastic sampler, namely an Euler--Maruyama discretization with $100$ time steps and stochasticity parameter $\xi = 4$.

  \item \textbf{Double precision.} To mitigate the accumulation of round-off errors, all computations during sampling are carried out in double precision.

  \item \textbf{Clipping of the log-densities.} The accumulated log-densities $\log p_j$ are clipped from above at every time step, which prevents overflow in the evaluation of $h_0$, $h_1$ and their derivatives.

  \item \textbf{Recentering of the log-weights.} The log-weights are recentered by their mean at every time step, which prevents them from attaining extreme values. This recentering is valid since the normalized weights are obtained by applying a softmax to the log-weights, and the softmax is invariant under the addition of a constant to all of its elements.
\end{enumerate}

Before applying the GFKC algorithm, we first examine the two agents in isolation. Figure~\ref{fig:MNIST_agents_samples} shows images generated independently by each agent using its own diffusion sampler, without any reweighting or resampling. As expected, the first agent produces images of the digits $1$ and $6$, while the second produces images of the digits $1$ and $7$. These samples serve as a reference for the results of the GFKC algorithm presented below.

\begin{figure}[htbp]
    \centering
    \begin{subfigure}[b]{0.48\textwidth}
        \centering
        \includegraphics[width=\textwidth]{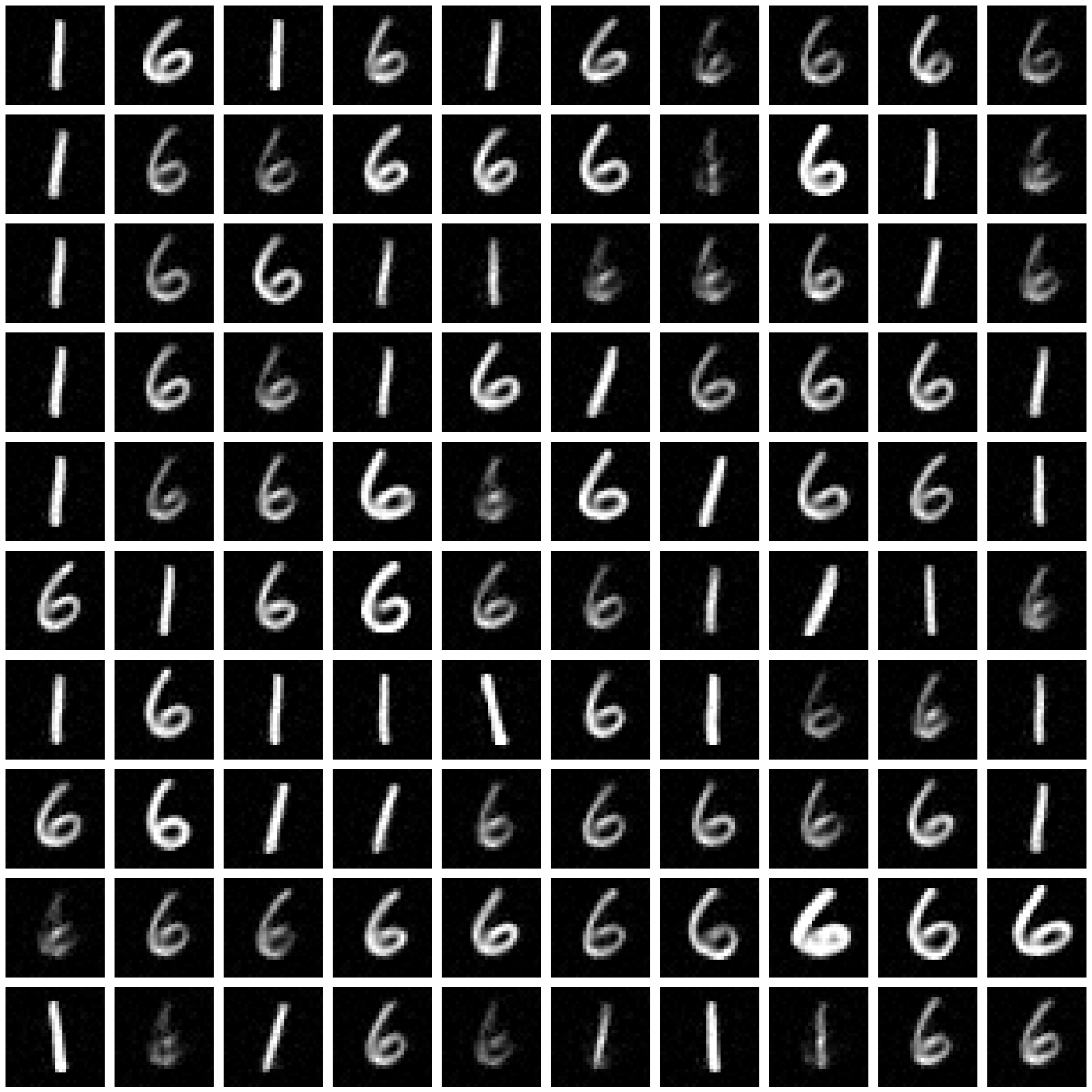}
        \label{fig:left1}
    \end{subfigure}
    \hfill
    \begin{subfigure}[b]{0.48\textwidth}
        \centering
        \includegraphics[width=\textwidth]{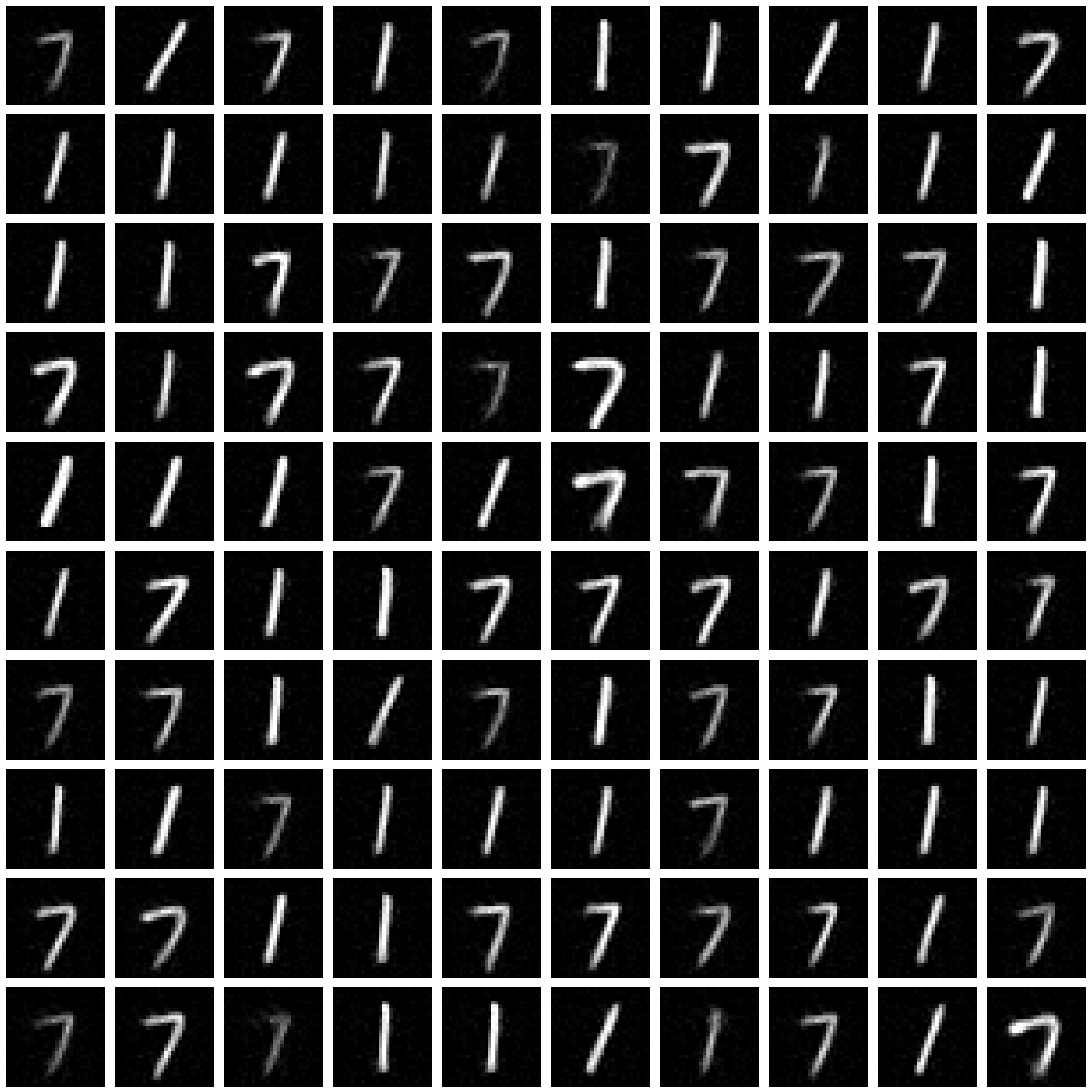}
        \label{fig:right1}
    \end{subfigure}
    \caption{Samples generated independently by each agent, without the GFKC algorithm. Left: $100$ images generated by the first agent, trained on the digits $\mathcal{S}_1$. Right: $100$ images generated by the second agent, trained on the digits $\mathcal{S}_2$.}
    \label{fig:MNIST_agents_samples}
\end{figure}

To demonstrate the GFKC algorithm in this setting, we consider the ratio target function $h_1(p_1,p_2)$ with~\eqref{eqn:ratio_target}, with $C = 100$, which favors regions of high density under the first agent and low density under the second, and hence the digit $\mathcal{S}_1\setminus\mathcal{S}_2=\{6\}$, which only the first agent has seen. In a parallel experiment, we use the same function with its arguments swapped, $h_1(p_2, p_1)$, which by symmetry favors the digit $\mathcal{S}_2\setminus\mathcal{S}_1=\{7\}$, exclusive to the second agent. Figure~\ref{fig:MNIST_ratio} shows the results of these two experiments, while Figure~\ref{fig:MNIST_weights_histograms} shows the histograms of the resulting log-weights.

\begin{figure}[htbp]
    \centering
    \begin{subfigure}[b]{0.48\textwidth}
        \centering
        \includegraphics[width=\textwidth]{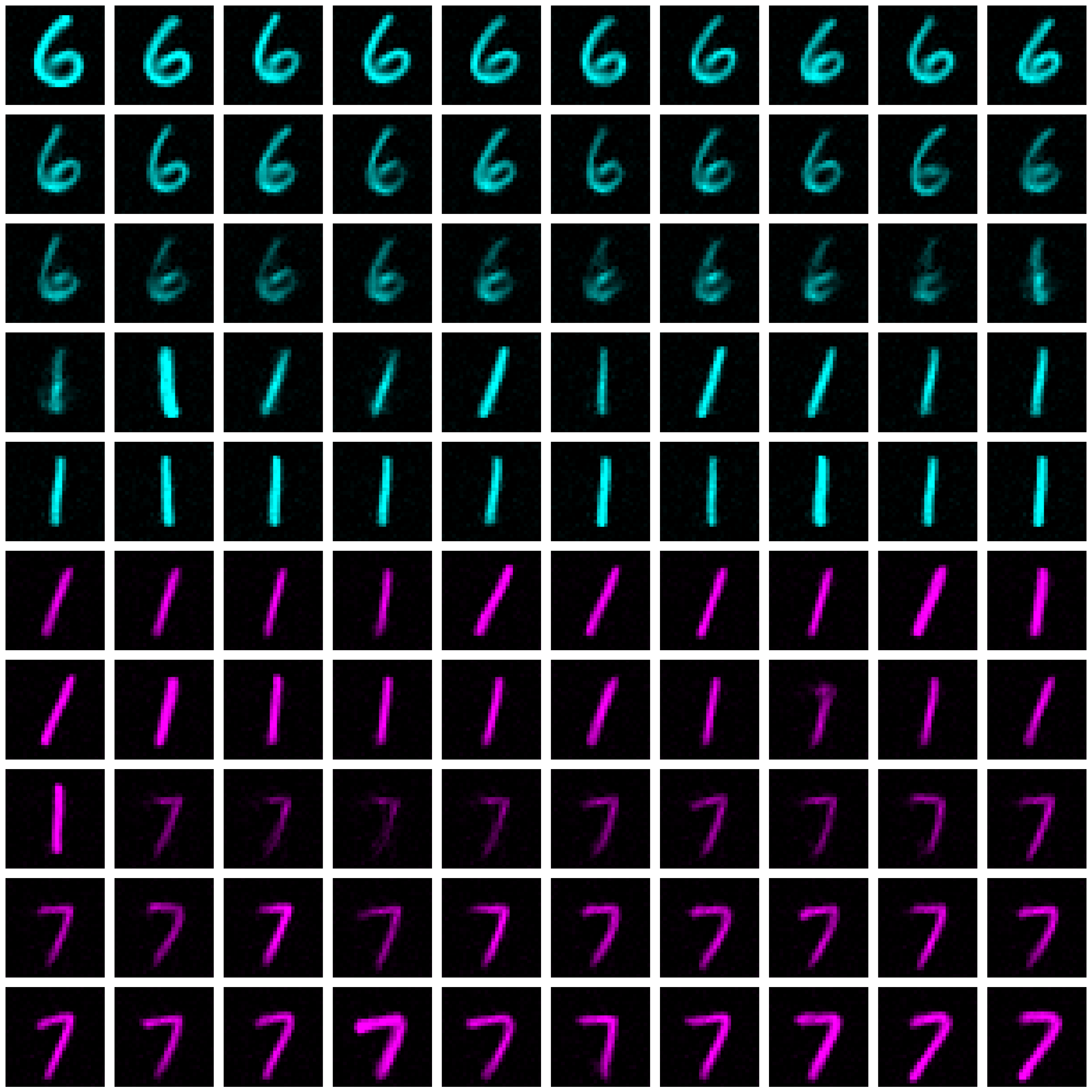}
    \end{subfigure}
    \hfill
    \begin{subfigure}[b]{0.48\textwidth}
        \centering
        \includegraphics[width=\textwidth]{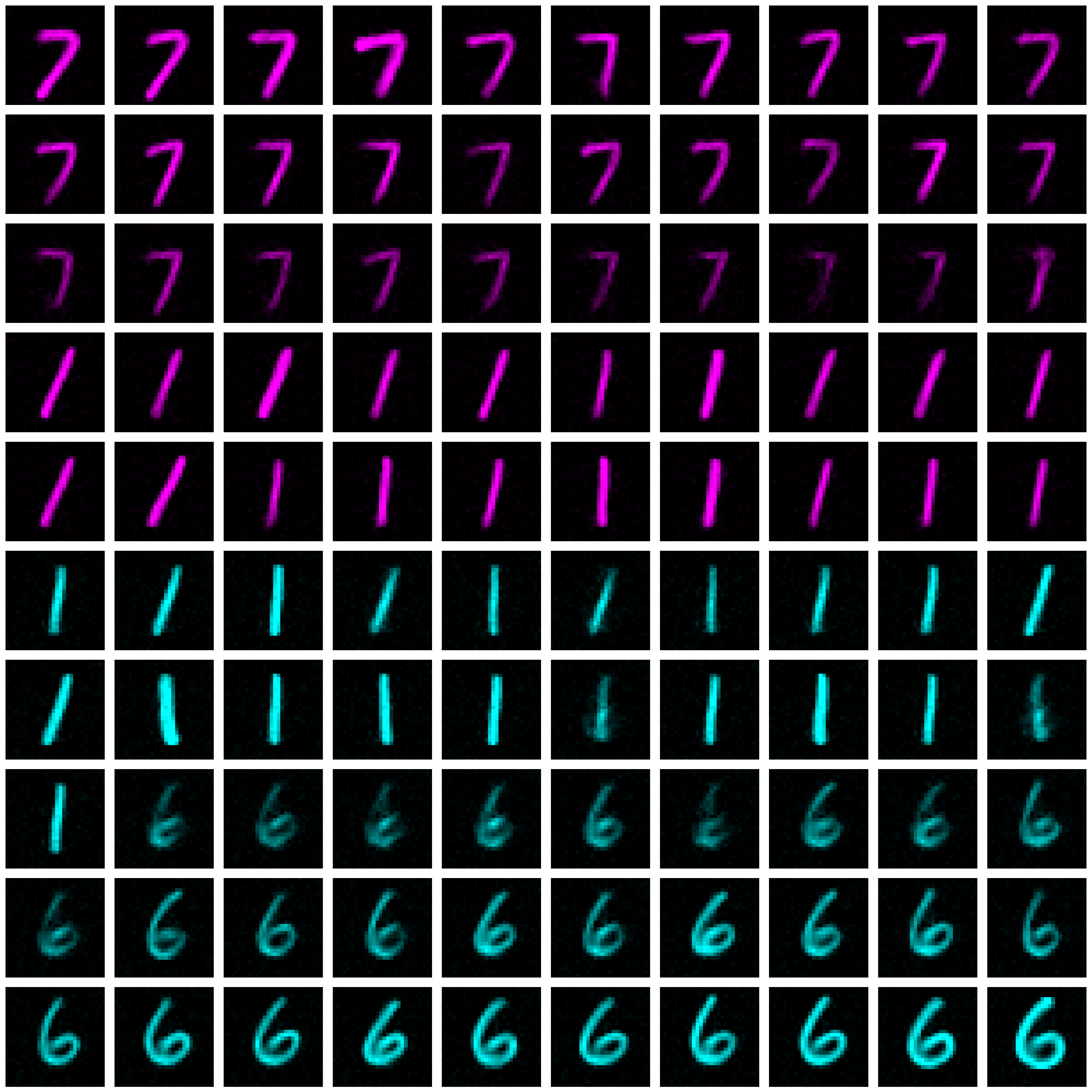}
    \end{subfigure}
    \caption{Results of the GFKC algorithm with the ratio target~\eqref{eqn:ratio_target}. Left: $h_1(p_1, p_2)$. Right: $h_1(p_2, p_1)$. Each panel displays the $100$ generated particles, $50$ per agent, sorted by their final log-weights in decreasing order, from the top-left to the bottom-right corner. Particles from the first agent (digits $1$ and $6$) are shown in cyan, and particles from the second agent (digits $1$ and $7$) in magenta. In the left panel, the highest weights are assigned to the digit $6$, exclusive to the first agent, followed by the shared digit $1$, while the digit $7$ receives the lowest weights. The right panel exhibits the symmetric behavior, with the digit $7$ ranked highest and the digit $6$ lowest.}
    \label{fig:MNIST_ratio}
\end{figure}

\begin{figure}[htbp]
    \centering
    \begin{subfigure}[b]{0.48\textwidth}
        \centering
        \includegraphics[width=\textwidth]{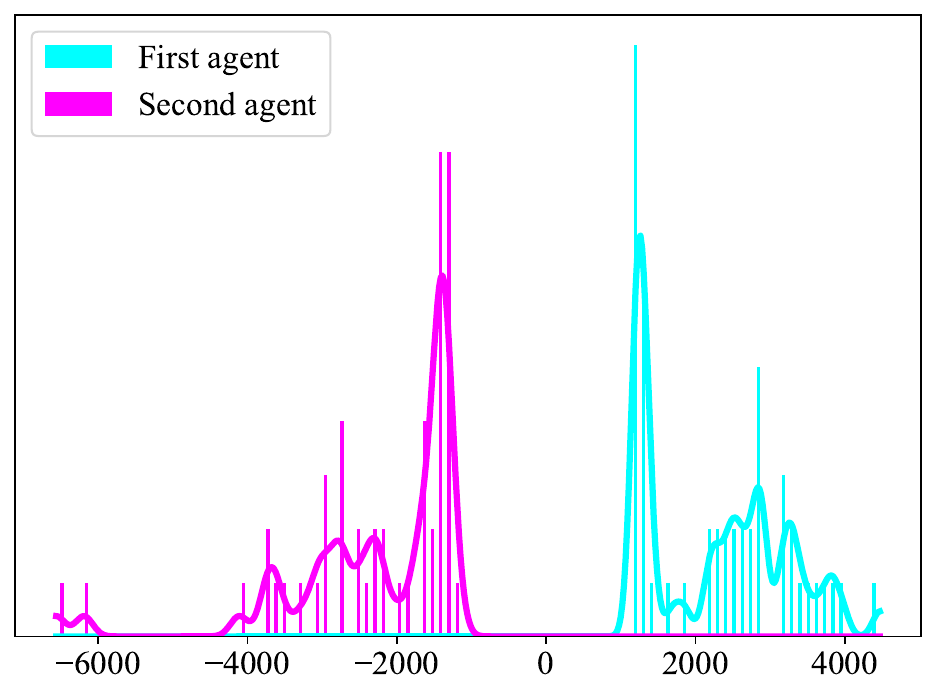}
    \end{subfigure}
    \hfill
    \begin{subfigure}[b]{0.48\textwidth}
        \centering
        \includegraphics[width=\textwidth]{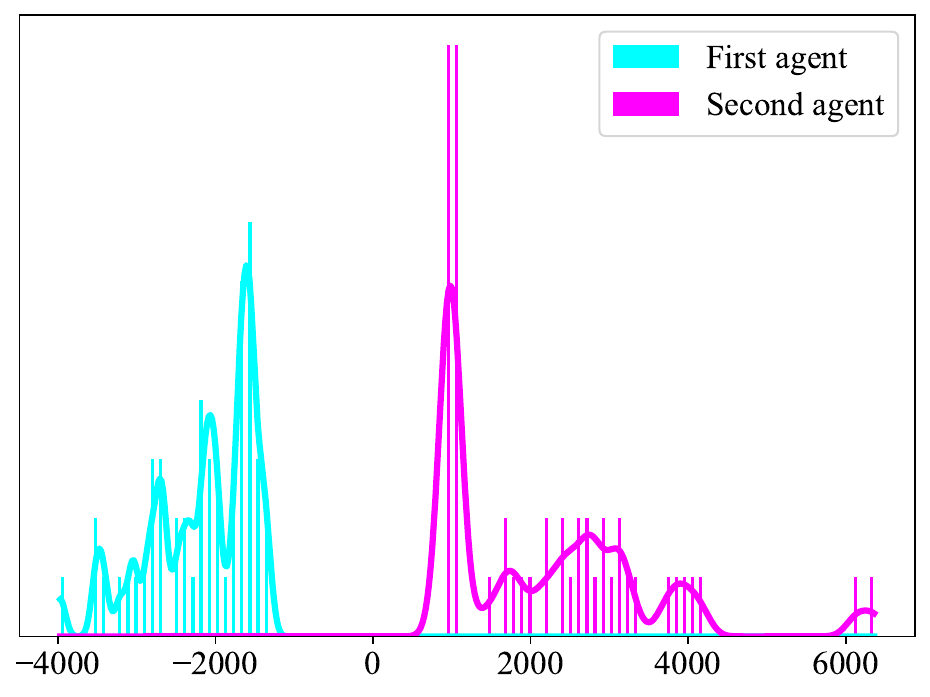}
    \end{subfigure}
    \caption{Histograms of the final log-weights of the particles generated by the GFKC algorithm, with the ratio target $h_1(p_1, p_2)$ (left) and the swapped target $h_1(p_2, p_1)$ (right). Colors correspond to the agents as in Figure~\ref{fig:MNIST_ratio}: cyan for the first agent (digits $1$ and $6$) and magenta for the second agent (digits $1$ and $7$). Solid curves show kernel density estimates of each histogram. In both cases, the log-weights of the two agents are clearly separated, with the agent holding the favored exclusive digit receiving the higher weights. Since the histograms are over log-weights, differences of several thousands translate into ratios of order $e^{1000}$ or more between the weights themselves, so that, after normalization, essentially all the mass is carried by the particles of the favored agent.}
    \label{fig:MNIST_weights_histograms}
\end{figure}

We next consider a target function of a different nature, namely $h_1 = p_1 + p_2$. Since the digit $1$ is assigned high density by both agents, $\mathcal{S}_1\cap\mathcal{S}_2=\{1\}$, this target places the largest mass on the shared digit, in contrast to the ratio target, which suppressed it. To further emphasize this effect, we sharpen the target by raising it to a power,
\begin{equation}~\label{eqn:MNIST_sum_target_sharpenned}
    h_1 = (p_1 + p_2)^{\gamma}, \qquad \gamma \geq 1,
\end{equation}
and examine how the distribution of the log-weights evolves as $\gamma$ increases.

Figure~\ref{fig:MNIST_sum} shows the results of the GFKC algorithm for $\gamma = 2$, while Figure~\ref{fig:MNIST_sum_sharpenning} illustrates how the histogram of the log-weights evolves for increasing values of $\gamma$. As $\gamma$ grows, the log-weights spread over an increasingly wide range, and the peak associated with the shared digit $1$ shifts toward larger values. This behavior follows directly from the structure of the target: since $\log h_1 = \gamma \log(p_1 + p_2)$, the contribution of $h_1$ to the weight increments in~\eqref{eqn:A}, \eqref{eqn:B} and~\eqref{eqn:E} is scaled by $\gamma$. Consequently, differences in log-weight between particles are amplified roughly in proportion to $\gamma$, and after normalization the weights concentrate increasingly on the particles of highest density under $p_1 + p_2$.

\begin{figure}[htbp]
    \centering
    \includegraphics[width=0.52\linewidth]{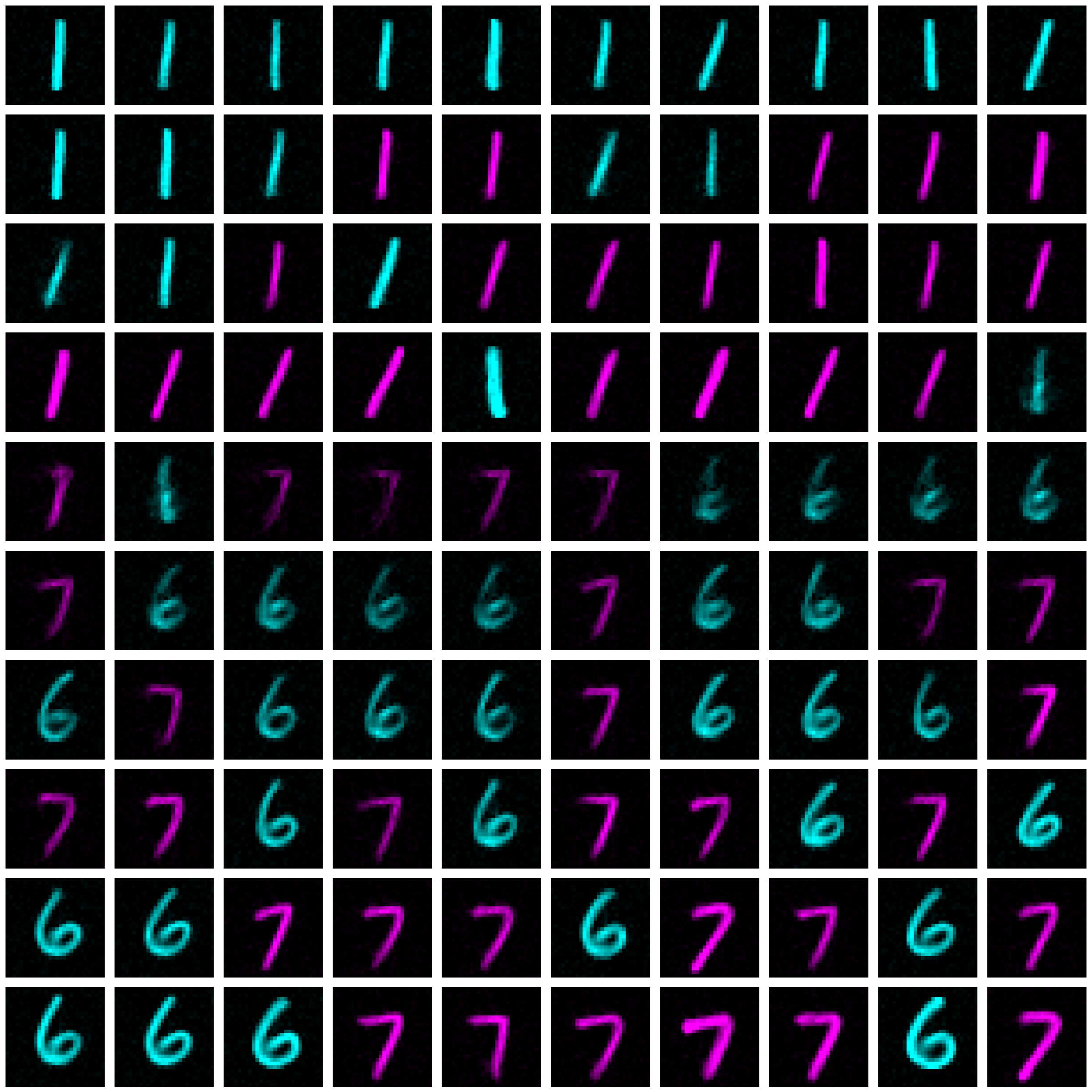}
    \caption{Results of the GFKC algorithm with the annealed sum target $h_1 = (p_1 + p_2)^2$ (i.e., $\gamma = 2$). The $100$ generated particles, $50$ per agent, are sorted by their final log-weights in decreasing order, from the top-left to the bottom-right corner, and colored by agent as in Figure~\ref{fig:MNIST_ratio}: cyan for the first agent (digits $1$ and $6$) and magenta for the second agent (digits $1$ and $7$). The highest weights are assigned to the shared digit $1$, generated by both agents, while the exclusive digits $6$ and $7$ are interleaved among the lower weights.}
    \label{fig:MNIST_sum}
\end{figure}

\begin{figure}[htbp]
    \centering
    \includegraphics[width=0.7\linewidth]{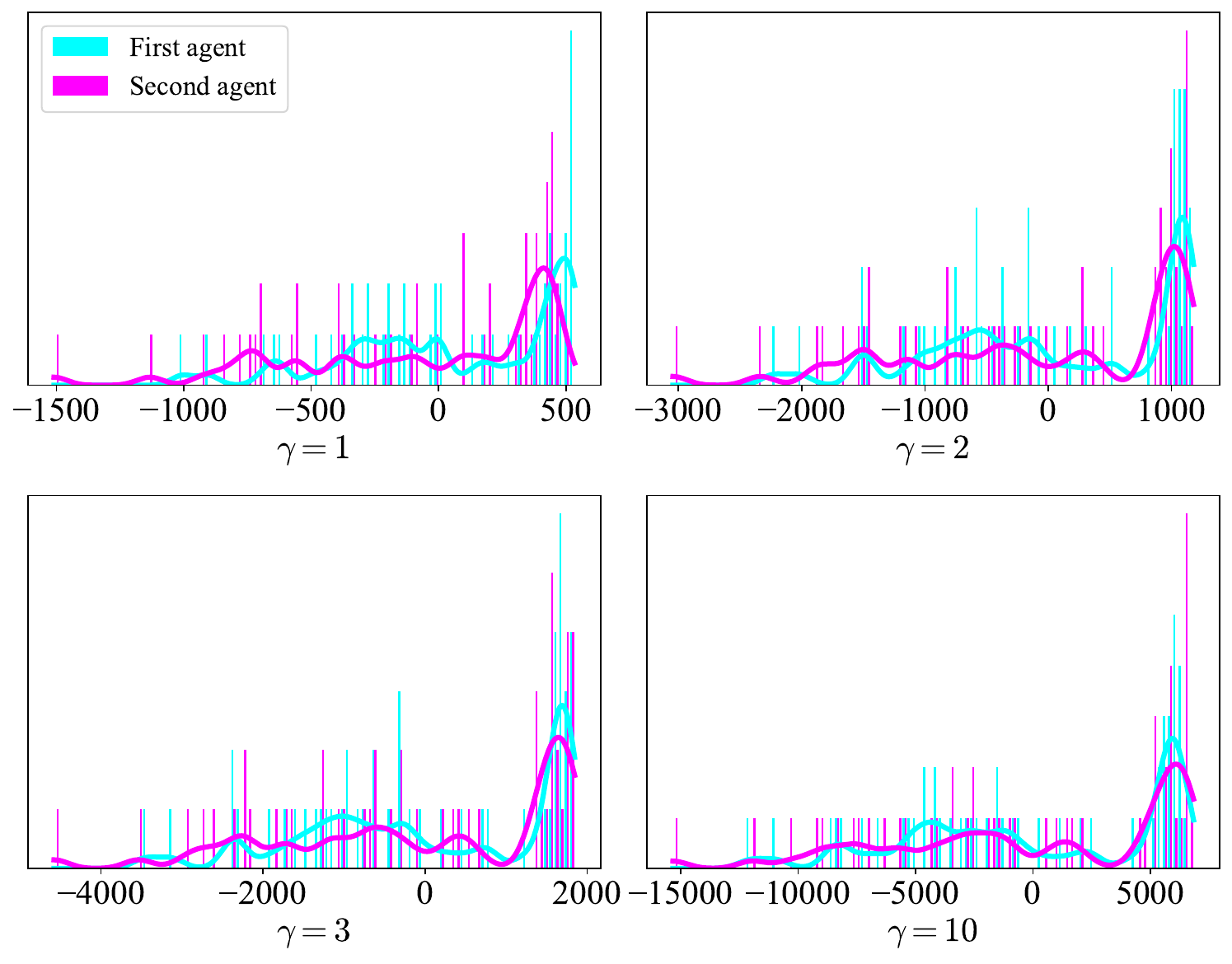}
    \caption{Histograms of the final log-weights for the sharpened sum target $h_1 = (p_1 + p_2)^{\gamma}$, with $\gamma \in \{1, 2, 3, 10\}$, colored by agent as in Figure~\ref{fig:MNIST_ratio}. As $\gamma$ increases, the range of the log-weights widens and the peak associated with the shared digit $1$ shifts toward larger values, so that the separation between the highest-weighted particles and the rest becomes increasingly pronounced. In contrast to the ratio target, the log-weights of the two agents overlap throughout, and the highest-weighted particles originate sometimes from the first agent and sometimes from the second, reflecting that both agents contribute to the shared digit.}
    \label{fig:MNIST_sum_sharpenning}
\end{figure}

We now demonstrate the effect of steering in the MNIST setting. To this end, we use the reward function
\begin{equation}~\label{eqn:MNIST_reward_function}
    r(x) = -\big\| \mathcal{R}(x) - \boxed{9}\;\big\|_2,
\end{equation}
where $\mathcal{R}$ denotes the rotation of an image by $\pi$ radians, and $\boxed{9}$ is the mean image of the digit $9$ over the MNIST training set, with pixel intensities rescaled to $[-1, 1]$ as for the training data. As in the Gaussian setting, the terminal sample $\widehat{x}_1$ entering the steered term in~\eqref{eqn:A} is estimated by projecting each particle to time $t = 1$ with the diffusion model of its agent, at the same steering times $t_k = 0.02\, k$ for $k = 0, \dots, 50$. As the terminal function, we take the sum target $h_1 = p_1 + p_2$, i.e., $\gamma = 1$, which, as shown above, favors the shared digit $1$.

Since a $6$ rotated by $\pi$ radians closely resembles a $9$, the reward function $r$ of~\eqref{eqn:MNIST_reward_function} favors images of the digit $6$, in competition with the target $h_1$, which favors the digit $1$. We therefore expect that, for a sufficiently small steering parameter $\kappa$, the highest-weighted particles will no longer depict the digit $1$, but rather the digit $6$.

We begin by illustrating the steering effect for a small value of $\kappa=10^{-6}$, corresponding to strong steering. As shown in Figure~\ref{fig:MNIST_sum_steered}, although the target $h_1$ favors the digit $1$, the steering term dominates, and the highest-weighted particles depict the digit $6$. As the weights decrease, the generated $6$’s become increasingly deformed and their rotations bear less resemblance to $\boxed{9}$. Figure~\ref{fig:MNIST_best_vs_mean9} compares the highest-weighted sample, rotated by $\pi$ radians, with $\boxed{9}$.

\begin{figure}[htbp]
    \centering
    \includegraphics[width=0.52\linewidth]{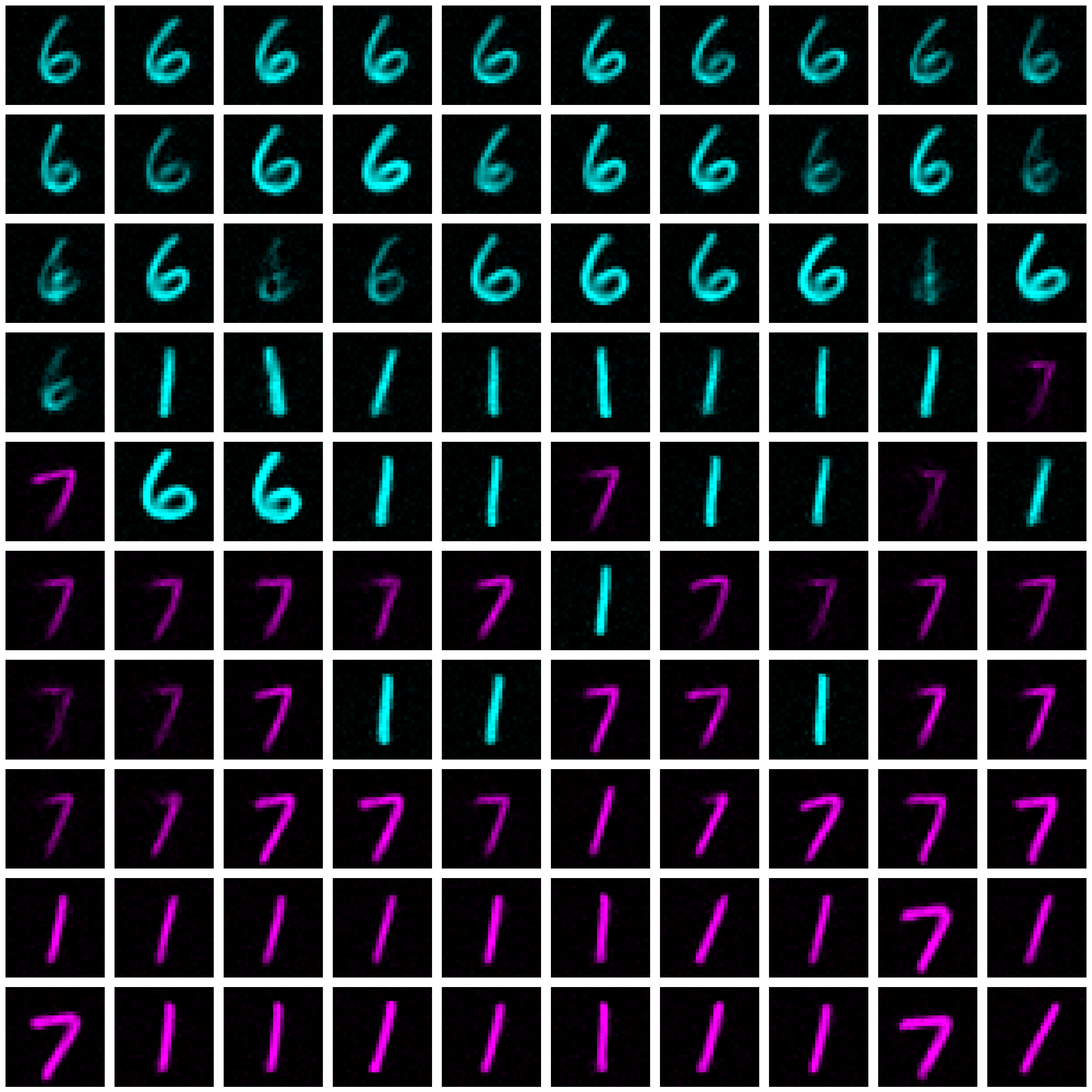}
    \caption{Results of the steered GFKC algorithm with the sum target $h_1$ of~\eqref{eqn:MNIST_sum_target_sharpenned} with $\gamma = 1$, the reward $r$ of~\eqref{eqn:MNIST_reward_function} and a small steering parameter $\kappa=10^{-6}$. The $100$ generated particles are sorted by their final log-weights in decreasing order, from the top-left to the bottom-right corner, and colored by agent as in Figure~\ref{fig:MNIST_ratio}. Although $h_1$ favors the shared digit $1$, the steering term dominates, and the highest weights are assigned to images of the digit $6$, all generated by the first agent. As the weights decrease, the generated $6$'s become increasingly deformed, and their $\pi$ radians rotations bear progressively less resemblance to the mean image $\boxed{9}$. They are followed by the digit $1$, while the particles of the second agent, which cannot generate the digit $6$, receive the lowest weights.}
    \label{fig:MNIST_sum_steered}
\end{figure}

\begin{figure}[htbp]
    \centering
    \includegraphics[width=0.52\linewidth]{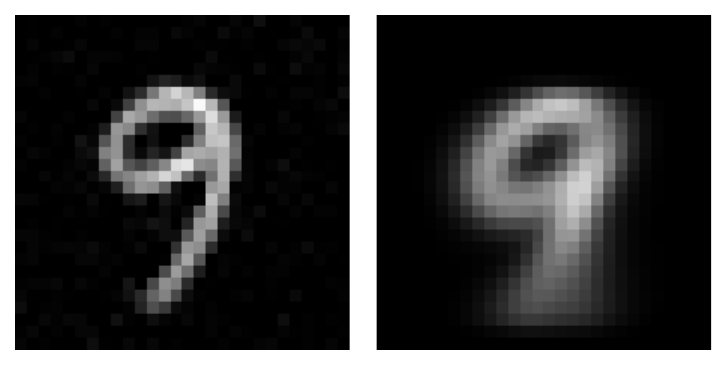}
    \caption{Left: the highest-weighted particle of the steered GFKC algorithm, a digit $6$ generated by the first agent, rotated by $\pi$ radians. Right: the mean image $\boxed{9}$ of the digit $9$ over the MNIST training set, which defines the reward $r$ of~\eqref{eqn:MNIST_reward_function}. The rotated sample closely resembles the mean image, illustrating how the steering guides the generation toward high values of the reward.}
    \label{fig:MNIST_best_vs_mean9}
\end{figure}

To further examine the relation between the reward and the weights, Figure~\ref{fig:MNIST_weights_reward_scatter} displays the reward values of the $100$ particles of Figure~\ref{fig:MNIST_sum_steered} against their final log-weights. The cluster in the top-right corner, combining the highest rewards with the highest log-weights, corresponds to the images of the digit $6$ generated by the first agent.

\begin{figure}[htbp]
    \centering
    \includegraphics[width=0.6\linewidth]{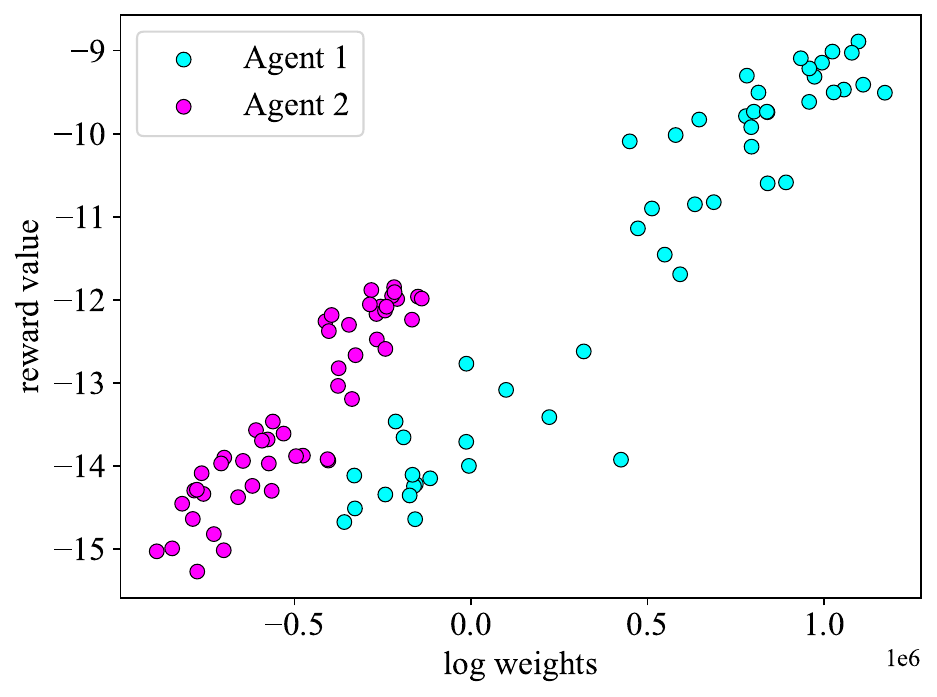}
    \caption{The reward values against the final log-weights for the $100$ particles of Figure~\ref{fig:MNIST_sum_steered}, colored by agent. The two quantities are strongly positively correlated, reflecting the dominance of the steering term for the small $\kappa=10^{-6}$. The cluster of roughly $30$ points in the top-right corner, combining the highest rewards with the highest log-weights, corresponds to the images of the digit $6$, all generated by the first agent. The particles of the second agent, which cannot generate the digit $6$, are concentrated in the lower-left region.}
    \label{fig:MNIST_weights_reward_scatter}
\end{figure}

We conclude this experiment by illustrating the steering effect for different values of $\kappa$. Analogously to Figure~\ref{fig:Gaussian_steering_effect} in the Gaussian setting, Figure~\ref{fig:MNIST_steering_effect} shows the effect of steering as $\kappa$ varies. However, since the density cannot be visualized in high dimension, we display, for each value of $\kappa$, the image with the highest log-weight. The values of $\kappa$ were chosen deliberately to capture the transition between the two competing effects, namely sampling the digit $1$, favored by the target $h_1$, and steering toward the digit $6$, favored by the reward. We observe that this transition occurs around the critical value $\kappa \approx 0.0015$.

\begin{figure}[htbp]
    \centering
    \includegraphics[width=\linewidth]{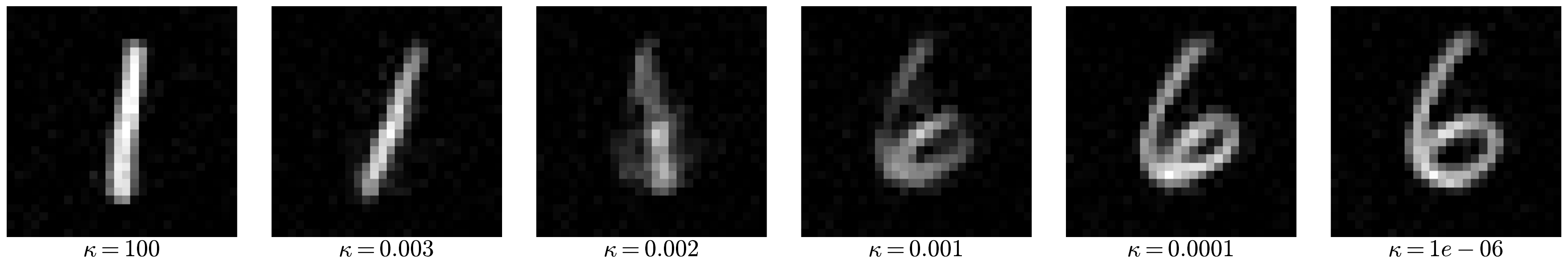}
    \caption{Highest-weighted image of the steered GFKC algorithm for decreasing values of $\kappa$. For large $\kappa$, the top image depicts the digit $1$, favored by $h_1$, while for small $\kappa$ the steering dominates and the digit $6$ emerges. The critical shift occurs between the third and fourth panels, at around $\kappa \approx 0.0015$.}
    \label{fig:MNIST_steering_effect}
\end{figure}

\section{Conclusion}\label{sec:conclusion}

In this work, we introduced the generalized Feynman--Kac correctors, a principled framework for sampling from arbitrary target functions of several pretrained diffusion models, with optional reward-based steering. Our main result in Theorem~\ref{thm:main_theorem} provides, in closed form, the Feynman--Kac potential of the geometric interpolation between an arbitrary initial and terminal function of the agents, and, combined with a generalized log-probability ODE, leads to Algorithm~\ref{alg:GFKC}, which is a practical algorithm that operates entirely in the log domain and corrects the generation on the fly, without retraining or fine-tuning. The numerical experiments confirm that the method accurately recovers non-trivial targets in an analytic setting, and that it scales to learned diffusion models in high dimension, where it isolates, combines and steers the concepts captured by the individual agents. Several directions remain open, including the reduction of weight degeneracy in high dimension, for instance through adaptive resampling or lower-variance divergence estimators, cheaper look-ahead estimates for steering, and the application of the framework to large-scale models and scientific domains such as molecular design.

\bibliographystyle{unsrtnat}
\bibliography{references}

\begin{appendices}
\section{}\label{app:lemma_proof}
Here we provide the proof of Lemma~\ref{lem:feynman_kac}. We rewrite Equation~\eqref{eqn:Feynman_Kac_diffusion_models} for the unnormalized without all of the explicit $x$ and $t$ dependencies:
\begin{equation*}
    \partial_t p = - \langle \nabla, pv \rangle + p g.
\end{equation*}

Now, call the corresponding partition function (i.e. normalization factor) of the unnormalized distribution $Z(t)$, and call the normalized version of the distribution $\widetilde{p}$. In this case,
\begin{equation*}
    \widetilde{p}(x, t) = \frac{p(x, t)}{Z(t)} \quad \Rightarrow \quad p(x, t) = Z(t) \widetilde{p}(x, t).
\end{equation*}
Plugging this into our equation gives $\partial_t (Z\tilde{p}) = - \langle \nabla, Z\tilde{p}  v \rangle + Z\tilde{p} g$, and therefore
\begin{align*}
  Z \partial_t \widetilde{p} + \widetilde{p} \partial_t Z = - Z \langle \nabla, \widetilde{p} v \rangle + Z \tilde{p} g \quad \text{and} \quad
  \partial_t \widetilde{p} = -\langle \nabla, \widetilde{p} v \rangle + \widetilde{p} \left( g - \tfrac{1}{Z} \partial_t Z \right).
\end{align*}

Now, using the fact that $Z(t)=\int p(x,t)\mathrm{d}x$, we have that
\begin{align*}
  \frac{\partial_t Z}{Z} = \frac{1}{Z} \int (\partial_t p)(x, t)\mathrm{d}x = \frac{1}{Z} \int \big(- \langle \nabla, pv \rangle + pg\big)\mathrm{d}x = - \int \langle \nabla, \widetilde{p} v \rangle \mathrm{d}x + \int \widetilde{p}g \,\mathrm{d}x.
\end{align*}
Note that due to the Divergence Theorem and the fact that $\widetilde{p}(x, t) \rightarrow 0$ as $x \rightarrow \infty$, the first term is $0$. Thus, we have that
\begin{equation*}
    \frac{\partial_t Z}{Z} = \int \widetilde{p}(x, t) g(x, t) \,\mathrm{d}x,
\end{equation*}
which overall gives
\begin{equation*}
    g(x, t) - \tfrac{1}{Z} \partial_t Z =  g(x, t) - \int \widetilde{p}(x, t) g(x, t) \,\mathrm{d}x =: \overline{g}(x, t).
\end{equation*}

Plugging this into the equation for $\partial_t \widetilde{p}$ gives
\begin{equation*}
    \partial_t \widetilde{p} = -\langle \nabla, \widetilde{p} v \rangle + \widetilde{p}\,\overline{g}.
\end{equation*}
However, we know that such an equation for the normalized distribution may be sampled from via the weighted SDE~\eqref{eqn:SDE_simulation}. This result is derived in Section 2 of~\cite{skreta2025feynman}.
\end{appendices}

\end{document}